\documentclass[twoside,11pt]{article}

\usepackage{jmlr2e}
\usepackage{balance} 
\usepackage{amsmath,amsfonts}
\usepackage{algorithmic}
\usepackage{algorithm}
\usepackage{array}
\usepackage[caption=false,font=normalsize,labelfont=sf,textfont=sf]{subfig}
\usepackage{textcomp}
\usepackage{stfloats}
\usepackage{url}
\usepackage{verbatim}
\usepackage{graphicx}
\usepackage{cite}
\usepackage[mode=buildnew]{standalone}
\usepackage{tikz}
\usepackage{kbordermatrix}
\usetikzlibrary{positioning, shapes, arrows, tikzmark,decorations.pathreplacing}
\tikzstyle{curly} = [decorate,decoration={brace,amplitude=10pt}]
\usepackage{mathrsfs}

\usepackage{graphicx}
\usepackage{algorithm}
\usepackage{algorithmic}

\newcommand \pol {\pi}

\newcommand \mdp {\mu}

\DeclareMathAlphabet{\mathpzc}{OT1}{pzc}{m}{it}

\if 1

\else

\fi

\usepackage{amsmath,array}

\usepackage{mathtools}
\usepackage{bm}
\usepackage[mode=buildnew]{standalone}
\usepackage{hyperref}
\usepackage{natbib}

\definecolor{wheat}{rgb}{0.96,0.87,0.70}
\definecolor{lightblue}{rgb}{0.8,0.8,1}
\definecolor{lightred}{rgb}{1,0.8,0.8}
\definecolor{lightgreen}{rgb}{0.8,1,0.8}

\ShortHeadings{Transfer RL with Maximum Likelihood Estimates}{Eriksson, Basu, Tram, Alibeigi and Dimitrakakis}
\firstpageno{1}

\begin{document}

\title{Reinforcement Learning in the Wild\\ with Maximum Likelihood-based Model Transfer}

\author{\name Hannes Eriksson \email hannes.eriksson@zenseact.com \\
       \addr Zenseact, Gothenburg, Sweden\\
       Chalmers University of Technology, Gothenburg, Sweden\\
       \AND
		\name Debabrota Basu \\
		\addr Scool, INRIA Lille- Nord Europe, Lille, France\\
		CRIStAL, CNRS, Lille, France\\
		 \AND
        \name Tommy Tram \\
       \addr Zenseact, Gothenburg, Sweden\\
       Chalmers University of Technology, Gothenburg, Sweden\\
	    \AND
	    \name Mina Alibeigi\\
	    \addr Zenseact, Gothenburg, Sweden\\
	    \AND
	    \name Christos Dimitrakakis\\
	    \addr University of Oslo, Oslo, Norway\\
	    University of Neuchatel, Neuchatel, Switzerland
	}

\editor{}

\maketitle

\begin{abstract}
In this paper, we study the problem of transferring the available Markov Decision Process (MDP) models to learn and plan efficiently in an unknown but similar MDP. We refer to it as \textit{Model Transfer Reinforcement Learning (MTRL)} problem. First, we formulate MTRL for discrete MDPs and Linear Quadratic Regulators (LQRs) with continuous state actions.
Then, we propose a generic two-stage algorithm, MLEMTRL, to address the MTRL problem in discrete and continuous settings. In the first stage, MLEMTRL uses a \textit{constrained Maximum Likelihood Estimation (MLE)}-based approach to estimate the target MDP model using a set of known MDP models. In the second stage, using the estimated target MDP model, MLEMTRL deploys a model-based planning algorithm appropriate for the MDP class.
Theoretically, we prove worst-case regret bounds for MLEMTRL both in realisable and non-realisable settings.
We empirically demonstrate that MLEMTRL allows faster learning in new MDPs than learning from scratch and achieves near-optimal performance depending on the similarity of the available MDPs and the target MDP.
\end{abstract}

\begin{keywords}
    Reinforcement Learning, Transfer Learning, Maximum Likelihood Estimation, Linear Quadratic Regulator
\end{keywords}

\section{Introduction}\label{sec:intro}

Deploying autonomous agents in the real world poses a wide variety of challenges. As in~\citet{dulac2021challenges}, we are often required to learn the real-world model with limited data, and use it to plan to achieve satisfactory performance in the real world. There might also be safety and reproducibility constraints, which require us to track a model of the real-world environment~\citep{skirzynski2021automatic}.
In light of these challenges, we attempt to construct a framework that can aptly deal with optimal decision making for a novel task, by leveraging external knowledge. As the novel task is unknown, we adopt the Reinforcement Learning (RL)~\citep{sutton2018reinforcement} framework to guide an agent's learning process and to achieve near-optimal decisions.

An RL agent interacts directly with the environment to improve its performance. Specifically, in model-based RL, the agent tries to learn a model of the environment and then use it to improve performance~\citep{moerland2020model}. In many applications, the depreciation in performance due to sub-optimal model learning can be paramount. For example, if the agent interacts with living things or expensive equipment, decision-making with an imprecise model might incur significant cost~\citep{polydoros2017survey}. In such instances, boosting the model learning by leveraging external knowledge from the existing models, such as simulators~\citep{peng2018sim}, physics-driven engines, etc., can be of great value~\citep{taylor2008transferring}. A model trained on simulated data may perform reasonably well when deployed in a new environment, given the novel environment is \emph{similar enough} to the simulated model. 
Also, RL algorithms running on different environments yield data and models that can be used to plan in another similar enough real-life environment.
In this work, we study the problem where we have access to multiple source models built using simulators or data from other environments, and we want to transfer the source models to perform efficient model-based RL in a different real-life environment.

\begin{example}
Let us consider that a company is designing autonomous driving agents for different countries in the world. The company has designed two RL agents that have learned to drive well in USA and UK. Now, the company wants to deploy a new RL agent in India. Though all the RL agents are concerned with the same task, i.e. driving, the models encompassing driver behaviors, traffic rules, signs, etc., can differ for each. For example, UK and India have left-handed traffic, while the USA has right-handed traffic.  However, learning a new controller \emph{specifically} for every new geographic location is computationally expensive and time-consuming, as both data collection and learning take time. Thus, the company might use the models learned for UK and USA, to estimate the model for India, and use it further to build a new autonomous driving agent (RL agent). Hence, being able to transfer the source models to the target environment allows the company to use existing knowledge to build an efficient agent faster and resource efficiently.\vspace*{-1em}
\end{example}
\textit{We address this problem of model transfer from source models to a target environment to plan efficiently.} We observe that this problem falls at the juncture of \emph{transfer learning} and \emph{reinforcement learning}~\citep{taylor2009transfer,lazaric2012transfer,laroche2017transfer}. 
\citet{lazaric2012transfer} enlists three approaches to transfer knowledge from the \emph{source tasks} to a \emph{target task}. (i) \emph{Instance transfer:} data from the source tasks is used to guide decision-making in the novel task~\citep{taylor2008transferring}. (ii) \emph{Representation transfer:} a representation of the task, such as learned neural network features, are transferred to perform the new task~\citep{zhang2018decoupling}. (iii) \emph{Parameter transfer:} the parameters of the RL algorithm or \emph{policy} are transferred~\citep{rusu2015policy}. In our paper, the source tasks are equivalent to the source models, and the target task is the target environment. Moreover, we adopt the \textbf{model transfer} approach (MTRL), which encompasses both (i) and (ii) (Section~\ref{sec:trl}). 

\citet{langley2006transfer} describes three possible benefits of transfer learning. The first is~\textbf{learning speed improvement}, i.e. decreasing the amount of data required to learn the solution. Secondly, \textbf{asymptotic improvement}, where the solution results in better asymptotic performance. Lastly, \textbf{jumpstart improvement}, where the initial proxy model serves as a better starting solution than that of one learning the true model from scratch. In this work, we propose a new algorithm to transfer RL that achieves both learning speed improvement and jumpstart improvement (Section~\ref{sec:experiments}). However, we might not find an asymptotic improvement in performance if compared with the best and unbiased algorithm in the true setting. Rather, we aim to achieve a model estimate that allows us to plan accurately in the target MDP (Section~\ref{sec:bounds}). 

\textbf{Contributions.} We aim to answer the two questions:\\
1. \emph{How can we accurately construct a model using a set of source models for an RL agent deployed in the wild?} \\
2. \emph{Does the constructed model allows efficient planning and yield improvements over learning from scratch?}

In this paper, we address these questions as follows:

1. \textit{A Taxonomy of MTRL:} First, we formulate the problem with the Markov Decision Processes (MDPs) setting of RL. We further provide a taxonomy of the problem depending on a discrete or continuous set of source models, and whether the target model is realisable by the source models (Section~\ref{sec:trl}).

2. \textit{Algorithm Design with MLE:} Following that, we design a two-stage algorithm MLEMTRL to plan in an unknown target MDP (Section~\ref{sec:algorithms}). In the first stage, MLEMTRL uses a Maximum Likelihood Estimation (MLE) approach to estimate the target MDP using the source MDPs. In the second stage, MLEMTRL uses the estimated model to perform model-based planning. We instantiate MLEMTRL for discrete state-action (tabular) MDPs and Linear Quadratic Regulators (LQRs). We also derive a generic bound on the goodness of the policy computed using MLEMTRL (Section~\ref{sec:bounds}).

3. \textit{Performance Analysis:} In Section~\ref{sec:experiments}, we empirically verify whether MLEMTRL improves the performance for unknown tabular MDPs and LQRs than learning from scratch. MLEMTRL exhibits learning speed improvement for tabular MDPs and LQRs. For LQRs, it incurs learning speed improvement and asymptotic improvement. We also observe that the more similar the target and source models are, the better the performance of MLEMTRL, as indicated by the theoretical analysis. 

Before elaborating on the contributions, we posit this work in the existing literature (Section~\ref{sec:related}) and discuss the background knowledge of MDPs and MLEs (Section~\ref{sec:background}).

\vspace*{-1em}\section{Related Work}\label{sec:related}
Our work on Model Transfer Reinforcement Learning is situated in the field of Transfer RL (TRL) and also is closely related to the multi-task RL and Bayesian multi-task RL literature. In this section, we elaborate on these connections.

TRL is widely studied in Deep Reinforcement Learning. \citet{zhu2020transfer} introduces different ways of transferring knowledge, such as \emph{policy transfer}, where the set of source MDPs $\mathcal{M}_s$ has a set of expert policies associated with them. The expert policies are used together with a new policy for the novel task by transferring knowledge from each policy. \citet{rusu2015policy} uses this approach, where a student learner is combined with a set of teacher networks to guide learning in multi-task RL. \citet{parisotto2015actor} develops an actor-critic structure to learn ways to transfer its knowledge to new domains.   \citet{arnekvist2019vpe} invokes generalisation across Q-functions by learning a master policy. Here,\textit{ we focus on model transfer instead of policy}.

Another seminal work in TRL, by~\citet{taylor2009transfer} distinguishes between \emph{multi-task learning} and \emph{transfer learning}. Multi-task learning deals with problems where the agent aims to learn from a distribution over scenarios, whereas transfer learning makes no specific assumptions about the source and target tasks. Thus, in transfer learning, the tasks could involve different state and action spaces, and different transition dynamics. Specifically, we focus on \textbf{model-transfer}~\citep{atkeson1997comparison} approach to TRL, where the state-action spaces and also dynamics can be different. \citet{atkeson1997comparison} performs model transfer for a target task with an identical transition model. Thus, the main consideration is to transfer knowledge to tasks with the same dynamics but varying rewards. \citet{laroche2017transfer} assumes a context similar to that of~\citet{atkeson1997comparison}, where the model dynamics are identical across environments. In our work, we rather assume that the reward function is the same, but the transition models are different. We believe this is an interesting question as the harder part of learning an MDP is learning the transition model. 
These works explicate a deep connection between the fields of \emph{multi-task learning} and \emph{TRL}. In general, TRL can be viewed as an extension of multi-task RL, where multiple tasks can either be learned simultaneously or have been learned \textit{a priori}. This flexibility allows us to learn even in settings where the state-actions and transition dynamics are different among tasks. \citep{rommel2017aircraft} describes a multi-task Maximum Likelihood Estimation procedure for optimal control of an aircraft. They identify a mixture of Gaussians, where the mixture is over each of the tasks. Here, we adopt an MLE approach to TRL in order to optimise performance for the target MDP (or a target task) than restricting to a mixture of Gaussians.

The Bayesian approach to multi-task RL~\citep{wilson2007multi, lazaric2010bayesian} tackles the problem of learning jointly how to act in multiple environments. \citet{lazaric2010bayesian} handles the \emph{open-world assumption}, i.e. the number of tasks is unknown. This allows them to transfer knowledge from existing tasks to a novel task, using value function transfer. However, this is significantly different from our setting, as we are considering model-based transfer. Further, \textit{we adopt an MLE-based framework in lieu of the full Bayesian procedure described in their work}.
In Bayesian RL, \citet{tamar2022regularization} also investigates a learning technique to generalise over multiple problem instances. By sampling a large number of instances, the method is expected to learn how to generalise from the existing tasks to a novel task. We do not assume access to such a prior or posterior distributions to sample from.

There is another related line of work, namely multi-agent transfer RL~\citep{da2019survey}. For example, \citet{liang2023federated} develops a TRL framework for autonomous driving using federated learning. They accomplish this by aggregating knowledge for independent agents. This setting is different from general transfer learning but could be incorporated if the source tasks are learned simultaneously with the target task. This requires cooperation among agents and is out of the scope of this paper.

\vspace*{-1em}\section{Background}\label{sec:background}
Here, we introduce the important concepts on which this work is based upon. Firstly, we introduce the way we model the dynamics of the tasks. Secondly, we describe the Maximum Likelihood Estimation framework used in this work.

\noindent\textbf{Markov Decision Process (MDP).}\label{sec:mdp}
We study sequential decision-making problems that can be represented as MDPs~\citep{puterman2014markov}. An MDP $\mdp = (\mathcal{S}, \mathcal{A}, \mathcal{R}, \mathcal{T}, \gamma)$ consists of a discrete or continuous state space denoted by $\mathcal{S}$, a discrete or continuous action-space $\mathcal{A}$, a reward function $\mathcal{R} : \mathcal{S} \times \mathcal{A} \rightarrow \mathbb{R}$ which determines the quality of taking action $a$ in state $s$, and a transition function $\mathcal{T} : \mathcal{S} \times \mathcal{A} \rightarrow \Delta(\mathcal{S})$ inducing a probability distribution over the successor states $s'$ given a current state $s$ and action $a$. Finally, in the infinite-horizon formulation, a discount factor $\gamma \in [0, 1)$ is assigned. The overarching objective for the agent is to compute a decision-making policy $\pol : \mathcal{S} \rightarrow \Delta(\mathcal{A})$ that maximises the expected sum of future discounted rewards up until the horizon $T$: $V_\mdp^{\pol}(s) = \mathbb{E}\Big[\sum_{t=0}^T \gamma^t \mathcal{R}(s_t, a_t)\Big]$. $V_\mdp^{\pol}(s)$ is called the value function of policy $\pol$ for MDP $\mdp$. Let $V_\mdp^{*} = V_\mdp^{\pol^{*}}$ denote the optimal value function. The technique used to obtain the the optimal policy $\pol^{*} = {\sup}_{\pol} \, V_\mdp^{\pol}$ depends on the MDP class. The MDPs with discrete state-action spaces are referred to as tabular MDPs. In this paper, we also study a class of MDPs with continuous state-action spaces, namely Linear Quadratic Regulators (LQRs)~\citep{kalman1960new}. In tabular MDPs, we employ \textsc{ValueIteration}~\citep{puterman2014markov} for model-based planning, whereas in the LQR setting, we use \textsc{RiccatiIteration}~\citep{willems1971least}. 

The standard metric used to measure the performance of a policy $\pol$~\citep{bell1982regret} for an MDP $\mdp$ is \textit{regret} $R(\mdp, \pol)$. Regret is the difference between the optimal value function and the value function of $\pol$. In this work, we extend the definition of regret for MTRL, where the optimality is taken for a policy class in the target MDP.

\noindent\textbf{Maximum Likelihood Estimation (MLE).}
One of the most popular methods of constructing point estimators is the \emph{Maximum Likelihood Estimation}~\citep{casella2021statistical} framework. Given a density function $f(x \, | \, \theta_1, \hdots, \theta_n)$ and associated i.i.d. data $X_1, \hdots, X_t$, the goal of the MLE scheme is to maximise, $\ell(\theta \, | \, x) \triangleq \ell(\theta_1, \hdots, \theta_n \, | \, x_1, \hdots, x_t) \triangleq \log \prod_{i=1}^t f(x_i \, | \, \theta_1, \hdots, \theta_n)$.
$\ell(\cdot)$ is called the log-likelihood function. The set of parameters $\theta$ maximising $\ell(\theta \, | \, x)$ is called the \emph{maximum likelihood estimator} of $\theta$ given the data $X_1, \hdots, X_t$. MLE has many desirable properties that we leverage in this work. For example, the MLE satisfies \emph{consistency}, i.e. under certain conditions, it achieves optimality even for \emph{constrained} MLE. An estimator being consistent means that if the data $X_1, \hdots, X_t$ is generated by $f(\cdot \, | \, \theta)$ and as $t\rightarrow \infty$, the estimate almost surely converges to the true parameter $\theta$. \citep{kiefer1956consistency} shows that MLE admits the consistency property given the following assumptions hold. The model is \emph{identifiable}, i.e. the densities at two parameter values must be different unless the two parameter values are identical. Further, the parameter space is \emph{compact} and \emph{continuous}. Finally, if the log-density is \emph{dominated}, one can establish that MLE converges to the true parameter almost surely~\citep{newey1987asymmetric}.
For problems where the likelihood is unbounded, flat, or otherwise unstable, one may introduce a penalty term in the objective function. This approach is called \emph{penalised maximum likelihood estimation}~\citep{ciuperca2003penalized, ouhamma2022bilinear}. As we in our work are mixing over known parameters, we do not need to add regularisation to our objective to guarantee convergence.

In this work, we iteratively collect data and compute new point estimates of the parameters and use them in our decision-making procedure. In order to carry out MLE, a likelihood function has to be chosen. In this work, we investigate two such likelihood functions in Section~\ref{sec:algorithms}, one for each respective model class. 

\vspace*{-1em}\section{A Taxonomy of Model Transfer RL}\label{sec:trl}
Now, we formally define the Model Transfer RL problem and derive a taxonomy of settings encountered in MTRL.

\vspace*{-1em}\subsection{MTRL: Problem Formulation}\vspace*{-.5em} 
Let us assume that we have access to a set of source MDPs $\mathcal{M}_s \triangleq \{\mu_i\}_{i=1}^m$. The individual MDPs can belong to a finite or infinite but compact set depending on the setting. For example, for tabular MDPs with finite state-actions, this is always a finite set. Whereas for MDPs with continuous state-actions, the transitions can be parameterised by real-valued vectors/matrices, corresponding to an infinite but compact set.
Given access to $\mathcal{M}_s$, we want to find an optimal policy for an unknown target MDP $\mdp^*$ that we encounter while deploying RL in the wild.
At each step $t$, we use $\mathcal{M}_s$ and the data observed from the target MDP $D_{t-1} \triangleq \{ s_0, a_0, s_1, \ldots, s_{t-1}, a_{t-1}, s_t\}$ to construct an estimate of $\mdp^*$, say $\hat{\mdp}^t$. Now, we use $\hat{\mdp}^t$ to run a model-based planner, such as \textsc{ValueIteration} or \textsc{RiccatiIteration}, that leads to a policy $\pol^t$. After completing this planning step, we interact with the target MDP using $\pol_t$ that yields an action $a_t$, and leads to observing $s_{t+1}, r_{t+1}$. We update the dataset with these observations: $D_t \triangleq D_{t-1} \cup \{a_{t}, s_t\}$. Here, we assume that all the source and target MDPs share the same reward function $\mathcal{R}$. We do not put any restrictions on the state-action space of target and source MDPs. 

Our goal is to compute a policy ${\pol}^t$ that performs as close as possible with respect to the optimal policy $\pol^*$ for the target MDP as the number of interactions with the target MDP $t\rightarrow \infty$.
This allows us to define a notion of regret for MTRL: $R(\mdp^{*}, \pol_t) \triangleq V_{\mdp^{*}}^{*}-V_{\mdp^{*}}^{\pol_t}$. Here, $\pol_t$ is a function of the source models $\mathcal{M}_s$, the data collected from target MDP $D_t$, and the underlying MTRL algorithm.
The goal of an MTRL algorithm is to minimise $R(\mdp^{*}, \pol_t)$.
For the parametric policies $\pol_{\theta}$ with $\theta \in \Theta \subset \mathbb{R}^d$, we can specialise the regret further for this parametric family: $R(\mdp^{*}, \pol_{\theta_t}) = V_{\mdp^{*}}^{\pol{\theta^*}}-V_{\mdp^{*}}^{\pol_{\theta_t}}$. For example, for LQRs, we by default work with linear policies. We use this notion of regret in our theoretical and experimental analysis.


\vspace*{-1em}\subsection{Three Classes of MTRL Problems}\vspace*{-.5em}
We begin by illustrating MTRL using Figure~\ref{fig:trl}. In the figure, the source MDPs $\mathcal{M}_s$ are depicted in red. This green area is the convex hull spanned by the source models $\mathcal{C}(\mathcal{M}_s)$. The target MDP $\mdp^{*}$, the best representative within the convex hull of the source models $\mdp$, and the estimated MDP $\hat{\mdp}$ are shown in blue, yellow, and purple, respectively. 
If the target model is \emph{inside} the convex hull, we call it a \textbf{realisable} setting. whereas If the target model is outside (as in Figure~\ref{fig:trl}), then we have a \textbf{non-realisable} setting.

Figure~\ref{fig:trl} also shows that the total deviation of the estimated model from the target model depends on two sources of errors: (i) realisability, i.e. how far is the target MDP $\mu^*$ from the convex hull of the source models $\mathcal{C}(\mathcal{M}_s)$ available to us, and (ii) estimation, i.e. how close is the estimated MDP $\hat{\mdp}$ to the best possible representation $\mdp$ of the target MDP. In the realisable case, the realisability gap can be reduced to zero, but not otherwise. This approach allows us to decouple the effect of the expressibility of the source models and the goodness of the estimator.
\setlength{\textfloatsep}{4pt}%
\begin{figure}
    \centering
    \includegraphics[width=0.35\paperwidth]{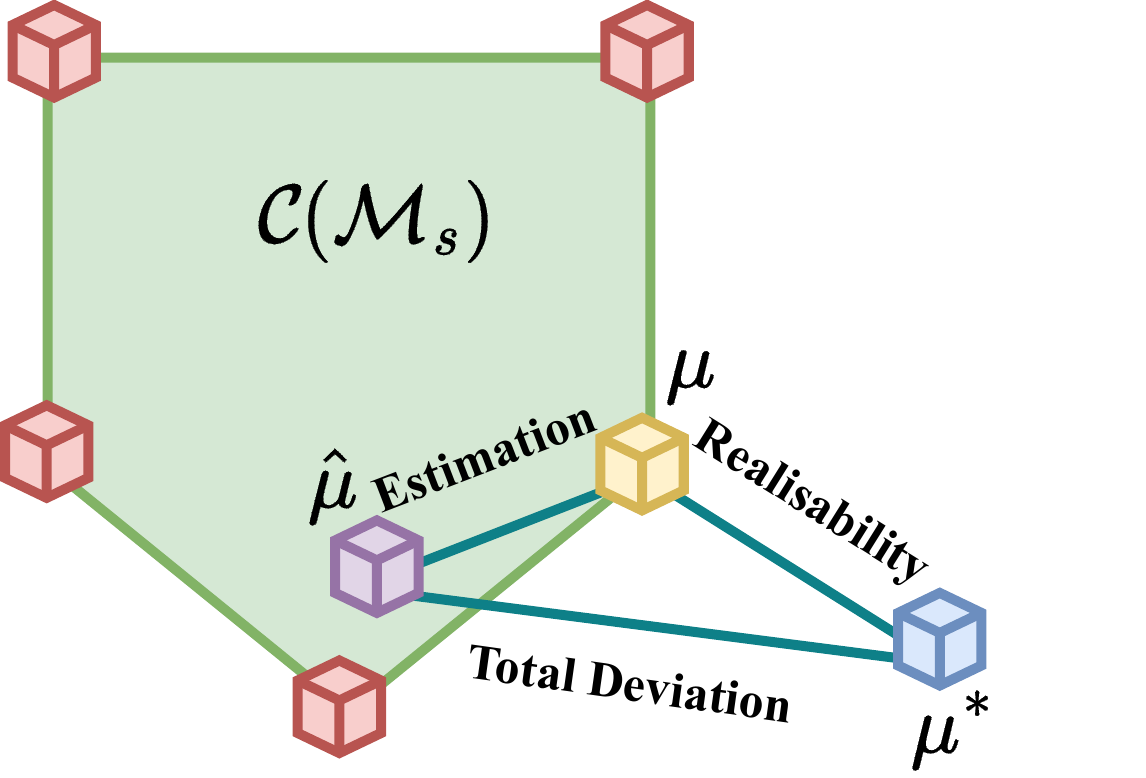}
    \caption{An illustration of the MTRL setting. The source models $\mathcal{M}_s$ are the red boxes. The green area is the convex hull $\mathcal{C}(\mathcal{M}_s)$ spanned by the source models. The target MDP $\mdp^{*}$ is displayed in blue, and the best proxy model is contained in the convex hull $\mdp$ in yellow. Finally, the estimator of the best proxy model $\hat{\mdp}$ is shown in purple.}
    \label{fig:trl}
\end{figure}

Now, we further elaborate on these three classes and the corresponding implications of performing MLE.

\textbf{I. Finite and Realisable Plausible Models.}\label{subsec:finite}
If the true model $\mdp^{*}$ is one of the target models, i.e. $\hat{\mdp} \in \mathcal{M}_s$, we have to identify the target MDP from a finite set of plausible MDPs. Thus, the corresponding MLE involves a finite set of parameters, i.e. the parameters of the source MDPs $\mathcal{M}_s$. We compute the MLE $\hat{\mdp}$ by solving the optimisation problem:
\begin{equation}
    \hat{\mdp} \in \underset{\mdp' \in \mathcal{M}_s}{\arg\max} \, \log \mathbb{P}(D_t \, | \, \mdp'), D_t \sim \mdp^{*}.
\end{equation}
This method may serve as a reasonable heuristic for the TRL problem, where the target MDP is the same as or reasonably close to one of the source MDPs. However, this method will potentially be sub-optimal if the target MDP is too different from the source MDPs. 
Even if $\mdp^{*}$ lies within the convex hull of the source MDPs (the green area in Figure~\ref{fig:trl}), this setting restricts the selection of a model to one of the red boxes. Thus, this setting fails to leverage the expressiveness of the source models as MLE allows us to accurately estimate models which are also in $\mathcal{C}(\mathcal{M}_s)$. Thus, we focus on the two settings described below.

\textbf{II. Infinite and Realisable Plausible Models.}\label{subsec:realisable}
In this setting, the target MDP $\mdp^{*}$ is in the convex hull $\mdp^{*} \in \mathcal{C}(\mathcal{M}_s)$ of the source MDPs. Thus, with respect to Class I, we extend the parameter space considered in MLE to an infinite but compact parameter set. 

Let us define the convex hull as $\mathcal{C}(\mathcal{M}_s) \triangleq \{\mdp_1 w_1 + \hdots + \mdp_m w_m \, | \, \mdp_i \in \mathcal{M}_s, w_i \geq 0, i=1, \hdots, m, \sum_{i=1}^m w_i = 1\}$. Then, the corresponding MLE problem with the corresponding likelihood function is given by:
\begin{equation}\label{eq:realisable}
    \hat{\mdp} \in \underset{\mdp' \in \mathcal{C}(\mathcal{M}_s)}{\arg\max} \, \log \mathbb{P}(D_t \, | \, \mdp'), D_t \sim \mdp^{*}.
\end{equation}
Since $\mathcal{C}(\mathcal{M}_s)$ induces a compact subset of model parameters $\mathcal{M}' \subset \mathcal{M}$, Equation~\eqref{eq:realisable} leads to a \emph{constrained maximum likelihood estimation problem}~\citep{aitchison1958maximum}. It implies that, if the parameter corresponding to the target MDP is in $\mathcal{M}'$, it can be correctly identified. In the case where the optimum lies inside, we can use constrained MLE to accurately identify the true parameters given enough experience from $\mdp^{*}$. This approach allows us to leverage the expressibility of the source models completely. However, $\mdp^{*}$ might lie outside or on the boundary. Either of these cases may pose problems for the optimiser.

\textbf{III. Infinite and Non-realisable Plausible Models.}\label{subsec:non-realisable}
This class is similar to Class II with the important difference that the true parameter $\mdp^*$ is outside the convex hull of source MDPs $\mathcal{C}(\mathcal{M}_s)$, and thus, the corresponding parameter is not in the induced parameter subset $\mathcal{M}'$. This key difference means the true parameters cannot be correctly identified. 
Instead, the objective is to identify the best proxy model $\mdp \in \mathcal{M}'$. The performance loss for using $\mdp$ instead of $\mdp^{*}$ is intimately related to the model dissimilarity $||\mdp^{*}-\mdp||_1$. This allows us to describe the limitation of expressivity of the source models by defining the \textit{realisability gap}: $\epsilon_{\mathrm{Realise}} \triangleq \min_{\mdp \in \mathcal{C}(\mathcal{M}_s)}\|\mdp^{*}-\mdp\|_1 $. The realisability gap becomes important while dealing with continuous state-action MDPs with parameterised dynamics, such as LQRs. 

\vspace*{-.8em}\section{MLEMTRL: MTRL with Maximum Likelihood Model Transfer}\label{sec:algorithms}
Now, we present the proposed algorithm, MLEMTRL. The algorithm consists of two stages, a \emph{model estimation} stage, and a \emph{planning} stage. After having obtained a plan, then the agent will carry out its decision-making in the environment to acquire new experiences. We sketch an overview of MLEMTRL in Algorithm~\ref{alg:mlemtrl}. For completeness, we also provide an extension to MLEMTRL called Meta-MLEMLTRL. This extension combines the MLEMTRL estimated model with the empirical model of the target task. This allows us to identify the true model even in the non-realisable setting. For brevity of space, further details are deferred to Appendix~\ref{sec:meta_mlemtrl}.

\begin{algorithm}[t!]
\caption{Maximum Likelihood Estimation for Model-based Transfer Reinforcement Learning (MLEMTRL)}\label{alg:mlemtrl}
\begin{algorithmic}[1]
\STATE \textbf{Input:} weights $\bm{w}^0$, $m$ source MDPs $\mathcal{M}_s$, data $D_0$, discount factor $\gamma$, iterations $T$.
\FOR{$t=0, \hdots, T$}
\STATE\textsc{// Stage 1: Model Estimation //}
\STATE $\bm{w}^{t+1}\leftarrow  \textsc{Optimiser}(\log\mathbb{P}(D_t \, | \, \Sigma_{i=1}^m w_i \mdp_i), \bm{w}^t)$\label{lin:optimiser}
\STATE Estimate the MDP: $\mdp^{t+1} = \sum_{i=1}^m w_i \mdp_i$
\STATE\textsc{// Stage 2: Model-based Planning //}
\STATE Compute the policy: $\pol^{t+1} \in \underset{\pol}{\arg\max} \, V_{\mdp^{t+1}}^\pol$
\STATE\textsc{// Control //}
\STATE Observe $s_{t+1}, r_{t+1} \sim \mdp^{*}(s_t, a_t), a_t\sim \pol^{t+1}(s_t)$
\STATE Update the dataset $D_{t+1} = D_t \cup \{s_t, a_t, s_{t+1}, r_{t+1}\}$
\ENDFOR
\STATE \textbf{return} An estimated MDP model $\mdp^T$ and a policy $\pol^T$
\end{algorithmic}
\end{algorithm}

\noindent\underline{\textbf{Stage 1: Model Estimation:}}
The first stage of the proposed algorithm is \emph{model estimation}. During this procedure, the likelihood of the data needs to be computed for the appropriate MDP class. In the tabular setting, we use a product of multinomial likelihoods, where the data likelihood is over the distribution of successor states $s'$ for a given state-action pair $(s, a)$. In the LQR setting, we use a linear-Gaussian likelihood, which is also expressed as a product over data observed from target MDP. 

\noindent\textbf{Likelihood for Tabular MDPs.}
The log-likelihood that we attempt to maximise in tabular MDPs is a product over $|\mathcal{S}|\times|\mathcal{A}|$ of pairs of multinomials, where $p_i$ is the probability of event $i$, $n^{s,a}$ is the number of times the state-action pairs $(s, a)$ appear in the data $D_t$, and $x_i^{s,a}$ is the number of times the state-action pair $(s, a, s_i)$ occurs in the data. That is, $\sum_{i=1}^{|\mathcal{S}|} x_i^{s,a} = n^{s,a}$. Specifically,
\begin{equation}
    \log \mathbb{P}(D_t \, | \, \bm{p}) = \log \Bigg(\prod_{s, a} n^{s,a}!\prod_{i=1}^{|\mathcal{S}|}\frac{p_i^{x^{s,a}_i}}{x^{s,a}_i!} \Bigg)
\end{equation}

\noindent\textbf{Likelihood for Linear-Gaussian MDPs.}
For continuous state-action MDPs, we use a linear-Gaussian likelihood. In this context, let $d_s$ be the dimensionality of the state-space, $\bm{s} \in \mathbb{R}^{d_s}$ and $d_a$ be the dimensionality of the action-space. Then, the mean function $\mathbf{M}$ is a $\mathbb{R}^{d_s}\times\mathbb{R}^{d_a+d_s}$ matrix. The mean visitation count to the successor state $\bm{s}_t'$ when an action $\bm{a}_t$ is taken at state $\bm{s}_t$ is given by $\mathbf{M}(\bm{a}_t, \bm{s}_t)$. We denote the corresponding covariance matrix of size $\mathbb{R}^{d_s}\times\mathbb{R}^{d_s}$ by $\mathbf{S}$. Thus, we express the log-likelihood by
\begin{align*}
\begin{aligned}
    \log \mathbb{P}(D_t \, | \, \mathbf{M}, \mathbf{S}) &= \log \prod_{i=1}^t \frac{\exp\Big(-\frac{1}{2}\bm{v}^\top\mathbf{S}^{-1}\bm{v}\Big)}{(2\pi)^{d_s/2}|\mathbf{S}|^{1/2}},\\
    &\textrm{where} \, \bm{s}_i'-\mathbf{M}(\bm{a}_i, \bm{s}_i) = \bm{v}.
\end{aligned}
\end{align*}

\noindent\textbf{Model Estimation as a Mixture of Models.}
As the optimisation problem involves weighing multiple source models together, we add a weight vector $\bm{w} \in [0, 1]^{m}$ with the usual property that $\bm{w}$ sum to $1$. This addition results in another outer product over the likelihoods shown above. Henceforth, $\mdp$ will refer to either the parameters associated with the product-Multinomial likelihood or the linear-Gaussian likelihood, depending on the model class.
\begin{equation}\label{eq:constrained_likelihood}
\begin{aligned}
\underset{\bm{w}}{\min} \quad &\log \mathbb{P}(D_t \, | \, \Sigma_{i=1}^m w_i \mdp_i), D_t \sim \mdp^{*}, \mdp_i\in\mathcal{M}_s,\\
\textrm{s.t.} \quad & \sum_{i=1}^m w_i = 1, w_i\geq0.\\
\end{aligned}
\end{equation}

Because of the constraint on $\bm{w}$, this is a constrained nonlinear optimisation problem. We can use any optimiser algorithm, denoted by \textsc{Optimiser}, for this purpose.

\noindent{\textsc{Optimiser}.} In our implementations, we use Sequential Least-Squares Quadratic Programming (SLSQP)~\citep{kraft1988software} as the \textsc{Optimiser}. SLSQP is a quasi-Newton method solving a quadratic programming subproblem for the Lagrangian of the objective function and the constraints.

Specifically, in Line~\ref{lin:optimiser} of Algorithm~\ref{alg:mlemtrl}, we compute the next weight vector $\bm{w}^{t+1}$ by solving the optimisation problem in Eq.~\eqref{eq:constrained_likelihood}. Let $f(\bm{w}) = \log \mathbb{P}(D_t \, | \, \Sigma_{i=1}^m w_i \mdp_i)$. Further, let $\lambda=\{\lambda_1,\hdots,\lambda_m\}$ and $\kappa$ be Lagrange multipliers. We then define the Lagrangian $\mathcal{L},$
\begin{equation}
    \mathcal{L}(\bm{w}, \lambda,\kappa) = f(\bm{w})-\lambda^\top\bm{w}-\kappa(1-\mathbf{1}^\top\bm{w}).
\end{equation}


Here, $\bm{w}^k$ is the $k$-th iterate. Finally, taking the local approximation of Eq.~\eqref{eq:constrained_likelihood}, we define the optimisation problem as:
\begin{equation}
    \begin{aligned}
        &\underset{\bm{d}}{\min} \, \frac{1}{2}\bm{d}^\top \nabla^2 \mathcal{L}(\bm{w},\lambda,\kappa)\bm{d}+\nabla f(\bm{w}^k)\bm{d}+f(\bm{w}^k)\\
        &\textrm{s.t.} \, \bm{d} + \bm{w}^k \geq 0, \mathbf{1}^\top\bm{w}^k = 1.
    \end{aligned}
\end{equation}
This minimisation problem yields the search direction $\bm{d}_k$ for the $k$-th iteration. Applying this iteratively and using the construction above ensures that the constraints posed in Eq.~\eqref{eq:constrained_likelihood} are adhered to at every step of MLEMTRL. At convergence, the $k$-th iterate, $\bm{w}^k$ is considered as the next $\bm{w}^{t+1}$ in Line~\ref{lin:optimiser} of Algorithm~\ref{alg:mlemtrl}.

\noindent\underline{\textbf{Stage 2: Model-based Planning:}}
When an appropriate model $\mdp^t$ has been identified at time step $t$, the next stage of the algorithm involves model-based planning in the estimated MDP. We describe two model-based planning techniques, \textsc{ValueIteration} and \textsc{RiccatiIteration} for tabular MDPs and LQRs, respectively.

\noindent\textsc{ValueIteration.}
Given the model, $\mdp^t$ and the associated reward function $\mathcal{R}$, the optimal value function of $\mdp^t$ can be computed iteratively as~\citep{sutton2018reinforcement}:
\begin{equation}\label{eq:value_iteration}
    V_{\mdp^t}^{*}(s) = \underset{a}{\max} \, \sum_{s'}\mathcal{T}_{s,s'}^a\Big(\mathcal{R}(s,a)+\gamma V_{\mdp^t}^{*}(s')\Big).
\end{equation}
The fixed-point solution to Eq.\ref{eq:value_iteration} is the optimal value function. When the optimal value function has been obtained, one can simply select the action maximising the action-value function. Let $\pol^{t+1}$ be the policy selecting the maximising action for every state, then $\pol^{t+1}$ is the policy the model-based planner will use at time step $t+1$.

\noindent\textsc{RiccatiIteration.}
A LQR-based control system, and thus, the corresponding MDP, is defined by four system matrices~\citep{kalman1960new}: $\mathbf{A}, \mathbf{B}, \mathbf{Q},\mathbf{R}$. The matrices $\mathbf{A, B}$ are associated with the transition model $\bm{s}_{t+1}-\bm{s}_t = \mathbf{A}\bm{s}_t + \mathbf{B}\bm{a}_t$. The matrices $\mathbf{Q, R}$ dictate the quadratic cost (or reward) of a policy $\pol$ under an MDP $\mdp$ is
\begin{equation*}
    V_{\mdp}^\pol = \sum_{t=0}^T \bm{s}_t^\top \mathbf{Q}\bm{s}_t + \bm{a}_t^\top \mathbf{R}\bm{a}_t.
\end{equation*}
Optimal policy is identified following~\citet{willems1971least} that states $\bm{a}_t = -\mathbf{K}\bm{s}_t$ at time $t$, where $\mathbf{K}$ is computed using $\mathbf{A}, \mathbf{B}, \mathbf{Q},\mathbf{R}$. We refer to Appendix~\ref{sec:riccati} for details.

\section{Theoretical Analysis}\label{sec:bounds}

In this section, we further justify the use of our framework by deriving worst-case performance degradation bounds relative to the optimal controller. The performance loss is shown to be related to the realisability of $\mdp^{*}$ under $\mathcal{C}(\mathcal{M}_s)$. In Figure~\ref{fig:trl}, we visualise the model dissimilarities, where $||\mdp-\hat{\mdp}||_1$ is the model estimation error, $||\mdp^{*}-\mdp||_1$ is the realisability gap and $||\mdp^{*}-\hat{\mdp}||_1$ the total deviation of the estimated model. Note that by the norm on MDP, we always refer to the $L_1$ norm over transition matrices.

\begin{theorem}[Performance Gap for Non-Realisable Models]\label{lemma:non-realisable}
Let $\mdp^* = (\mathcal{S}, \mathcal{A}, \mathcal{R}, \mathcal{T}^*, \gamma)$ be the true underlying MDP. Further, let $\mdp = (\mathcal{S}, \mathcal{A}, \mathcal{R}, \mathcal{T}, \gamma)$ be the maximum likelihood $\mdp \in \arg\min\nolimits_{\mdp' \in \mathcal{C}(\mathcal{M}_s)}~\mathbb{P}(D_\infty \, | \, \mdp'), D_\infty \sim \mdp^{*}$ and $\hat{\mdp} = (\mathcal{S}, \mathcal{A}, \mathcal{R}, \hat{\mathcal{T}}, \gamma)$ be a maximum likelihood estimator of $\mdp$. In addition, let $\pol^*, \pol, \hat{\pol}$ be the optimal policies for the respective MDPs. Then, if $\mathcal{R}$ is a bounded reward function $\forall_{(s, a)} \, r(s, a) \in [0, 1]$ and with $\epsilon_{\mathrm{Estim}}$ being the estimation error and $\epsilon_{\mathrm{Realise}} \triangleq \min_{\mdp \in \mathcal{C}(\mathcal{M}_s)}\|\mdp^{*}-\mdp\|_1 $ the realisability gap. Then, the performance gap is given by,
\begin{equation}
    ||V_{\mdp^*}^*-V_{\mdp^*}^{\hat{\pol}}||_\infty \leq \frac{3(\epsilon_{\mathrm{Estim}}+\epsilon_{\mathrm{Realise}})}{(1-\gamma)^2}.
\end{equation}

\end{theorem}
For the full proof, see Appendix~\ref{sec:proof_non_realisability}. This result is comparable to recent results such as~\citep{zhang2020multi} but here with an explicit decomposition into model estimation error and realisability gap terms.

\begin{remark}[Bound on $L_1$ Norm Difference in the Realisable Setting]\label{remark:l1_norm}
It is known~\citep{strehl2005theoretical, auer2008near, qian2020concentration} that in the realisable setting, it is possible to bound the model estimation error term $\epsilon_{\mathrm{Estim}}$ via the following argument. Let $\mdp^*$ be the true underlying MDP, and $\hat{\mdp}$ be an MLE estimate of $\mdp^*$, as defined in Theorem~\ref{lemma:non-realisable}. If $\mathcal{R}$ is a bounded reward function, i.e. $r(s, a) \in [0, 1], \forall{(s, a)}$, and $\epsilon_{\mathrm{Estim}}$ is upper bound on the $L_1$ norm between $\mathcal{T}^*$ and $\hat{\mathcal{T}}$. If $n^{s,a}$ be the number of times $(s,a)$ occur together, then with probability $1-SA\delta$,
\begin{equation*}
        ||\mathcal{T}^*-\hat{\mathcal{T}}||_1 \leq \epsilon_{\mathrm{Estim}} \leq \sum_{s\in\mathcal{S}}\sum_{a\in\mathcal{A}}\sqrt{\frac{2\log\big((2^S -2)/\delta)\big)}{n^{s,a}}} 
\end{equation*}
    From this, it can be said that the total $L_1$ norm then scales on the order of $\mathcal{O}(SA\sqrt{S+\log(1/\delta)}/\sqrt{T})$.
\end{remark}
This result is specific to tabular MDPs. In tabular MDPs, the maximum likelihood estimate coincides with the empirical mean model. Further details are in Appendix~\ref{sec:proof_remark}.

\begin{remark}[Performance Gap in the Realisable Setting]
A trivial worst-case bound for the realisable case (Section~\ref{subsec:realisable}) can be obtained by setting $\epsilon_{\mathrm{Realise}}=0$ because by definition of the realisable case $\mdp^{*} \in \mathcal{C}(\mathcal{M}_s)$.
\end{remark}

\section{Experiments}\label{sec:experiments}

To benchmark the performance of MLEMTRL, we compare ourselves to a posterior sampling method (\textbf{PSRL})~\citep{osband2013more}, equipped with a combination of product-Dirichlet and product-NormalInverseGamma priors for the tabular setting, and Bayesian Multivariate Regression prior~\citep{minka2000bayesian} for the continuous setting. In PSRL, at every round, a new model is sampled from the prior, and it learns in the target MDP from scratch. Finally, for model-based planning, we use \textsc{RiccatiIterations} to obtain the optimal linear controller for the sampled model. In the continuous action setting, we compare the performance to the baseline algorithm multi-task soft-actor critic (\textbf{MT-SAC})~\citep{haarnoja2018soft, yu2020meta} and a modified \textbf{MT-SAC-TRL} using data from the novel task during learning. In the tabular MDP setting, we compare against multi-task proximal policy optimisation (\textbf{MT-PPO)}~\citep{schulman2017proximal, yu2020meta} and similarly \textbf{MT-PPO-TRL}. 

The objectives of our empirical study are two-fold:

1. How does \textsc{MLEMTRL} impact performance in terms of \textbf{learning speed}, \textbf{jumpstart improvement} and \textbf{asymptotic convergence} compared to our baseline?

2. What is the performance loss of \textsc{MLEMTRL} in the \textbf{non-realisable setting}?

We conduct two kinds of experiments to verify our hypotheses. Firstly, in the upper row of Figure~\ref{fig:full_results}, we consider the realisable setting, where the novel task $\mdp^*$ is part of the convex hull $\mathcal{C}(\mathcal{M}_s)$. In this case, we are looking to identify an improvement in some or all of the aforementioned qualities compared to the baselines.
Further, in the bottom row of Figure~\ref{fig:full_results}, we investigate whether the algorithm can generalise to the case beyond what is supported by the theory in Section~\ref{subsec:realisable}. We begin by recalling the goals of the transfer learning problem~\citep{langley2006transfer}.


\noindent\textit{Learning Speed Improvement:}
A learning speed improvement would be indicated by the algorithm reaching its asymptotic convergence with less data.

\noindent\textit{Asymptotic Improvement:}
 An asymptotic improvement would mean the algorithm converges asymptotically to a superior solution to that one of the baseline.
 
\noindent\textit{Jumpstart Improvement:}
A jumpstart improvement can be verified by the behaviour of the algorithm during the early learning process. In particular, if the algorithm starts at a better solution than the baseline, or has a simpler optimisation surface, it may more rapidly approach better solutions with much less data.
\setlength{\textfloatsep}{2pt}%
\begin{figure*}[t!]
    \centering
    \includegraphics[width=0.24\textwidth]{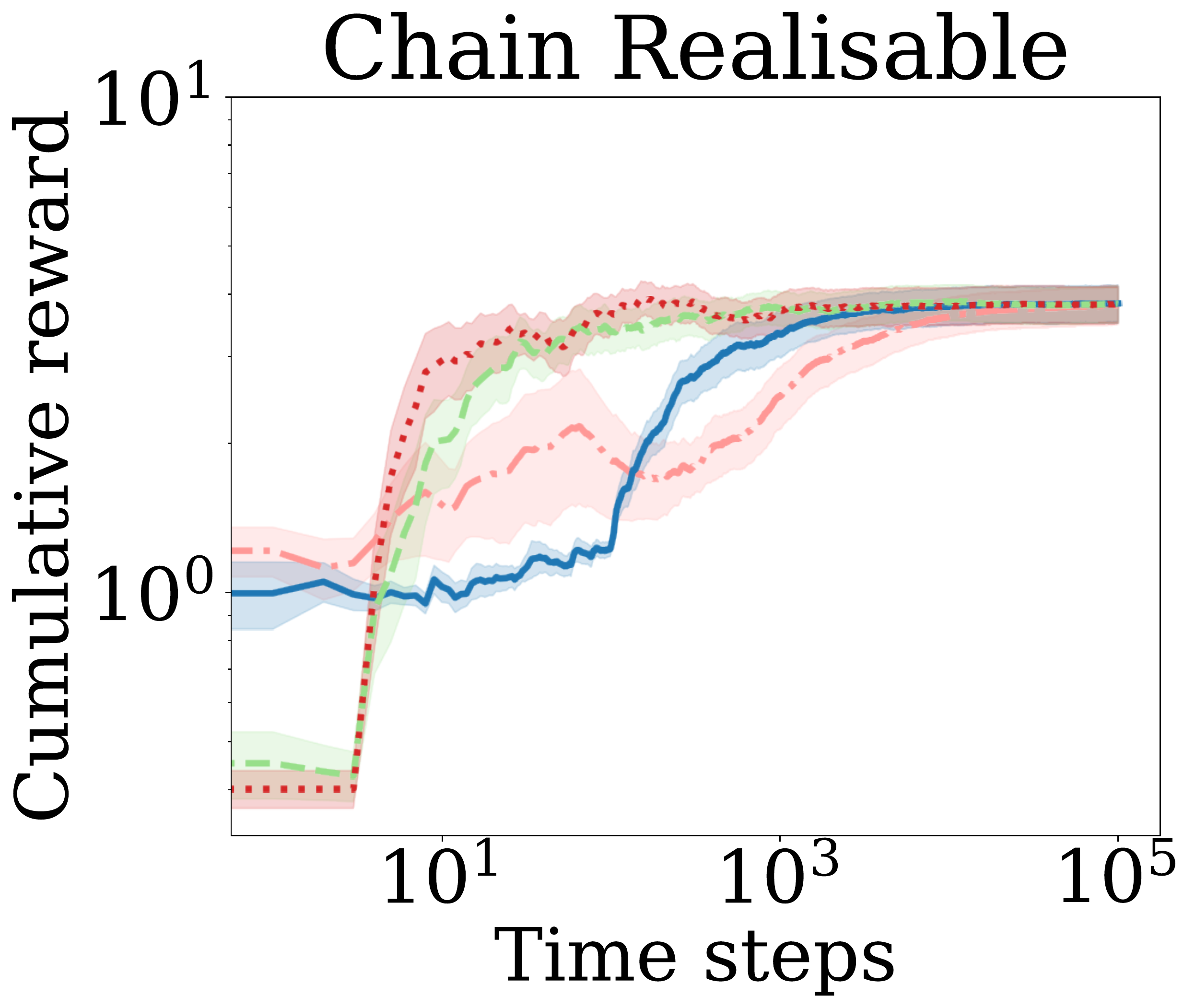}         
    \includegraphics[width=0.24\textwidth]{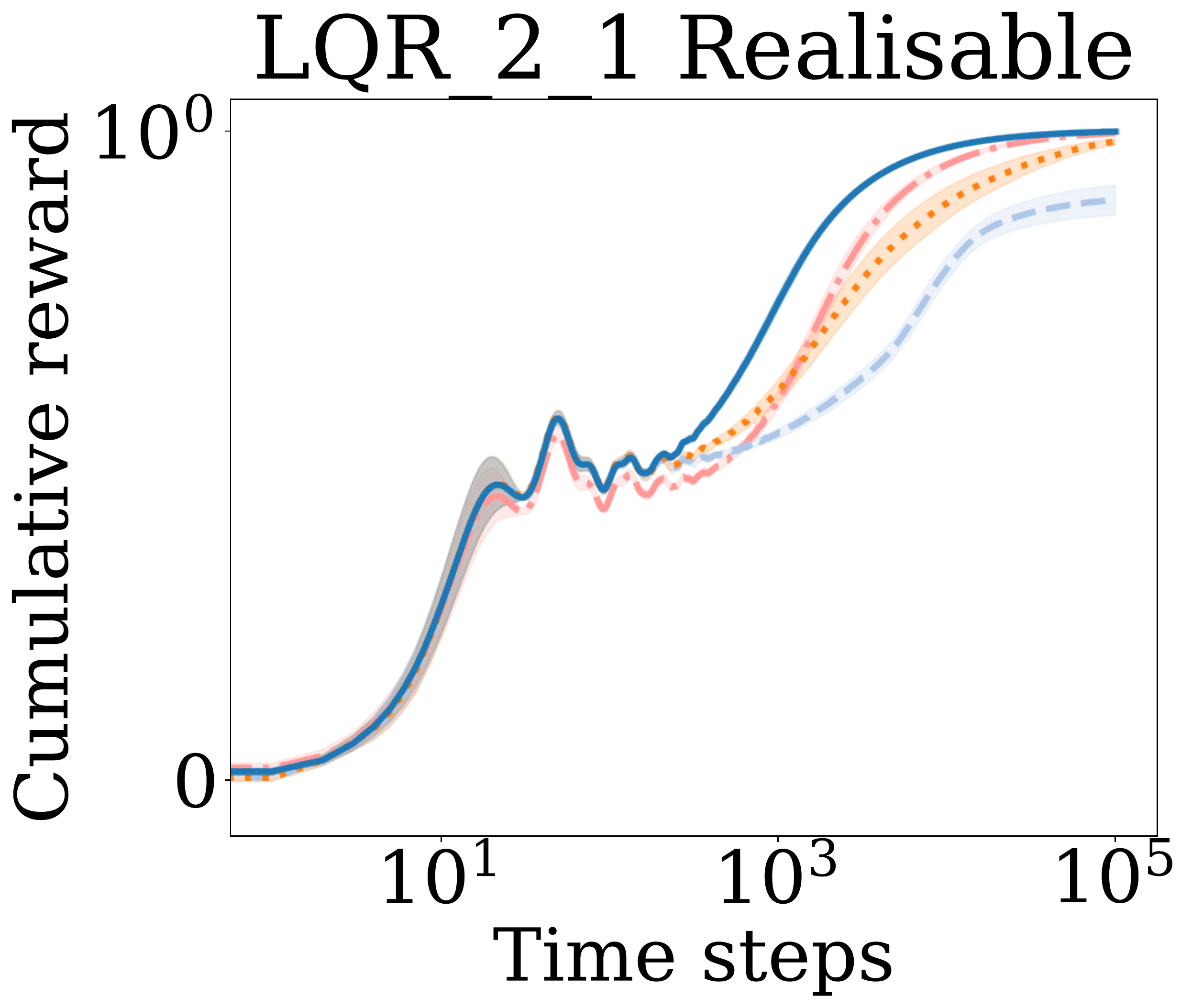} 
    \includegraphics[width=0.24\textwidth]{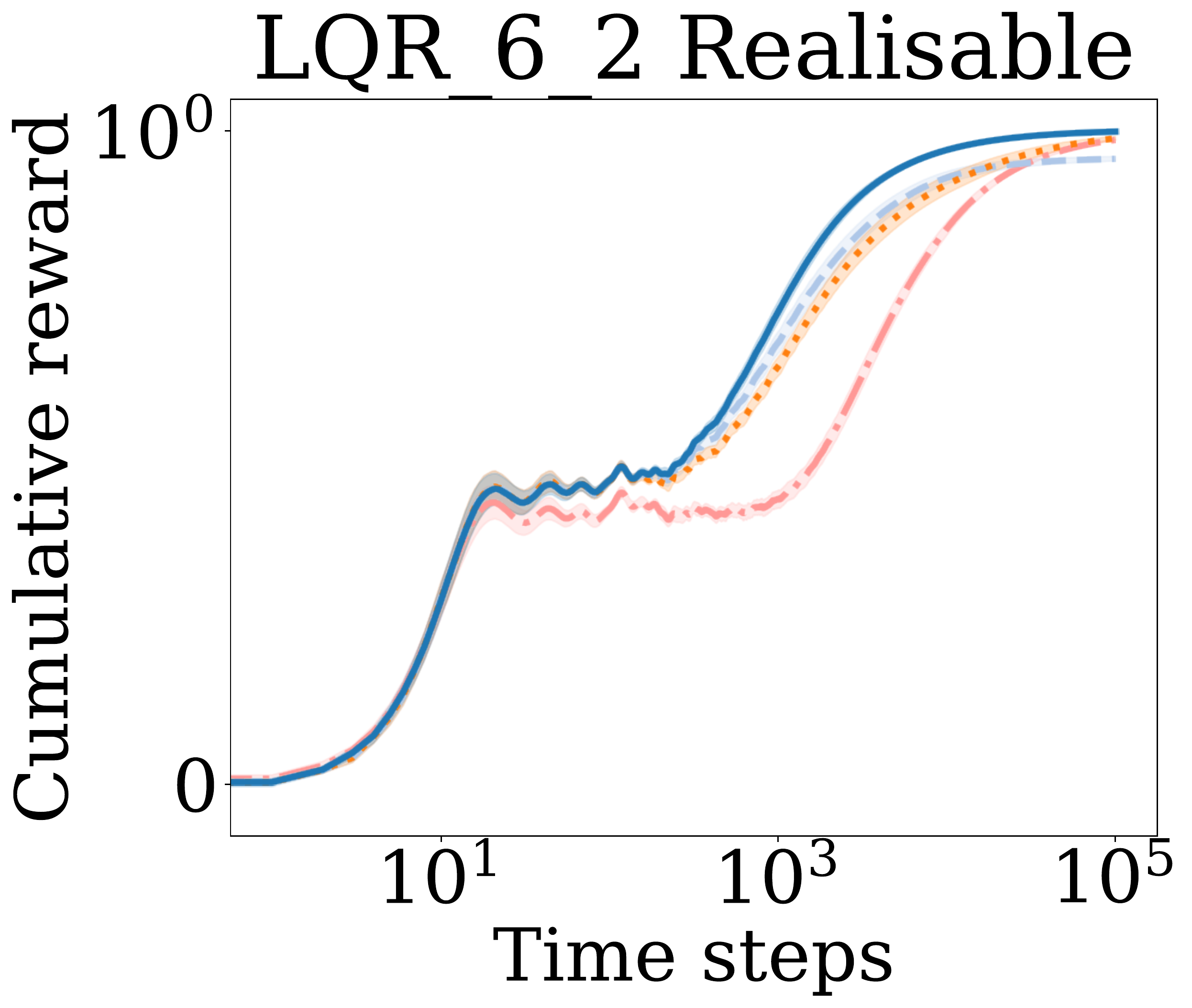}\\
    \includegraphics[width=0.24\textwidth]{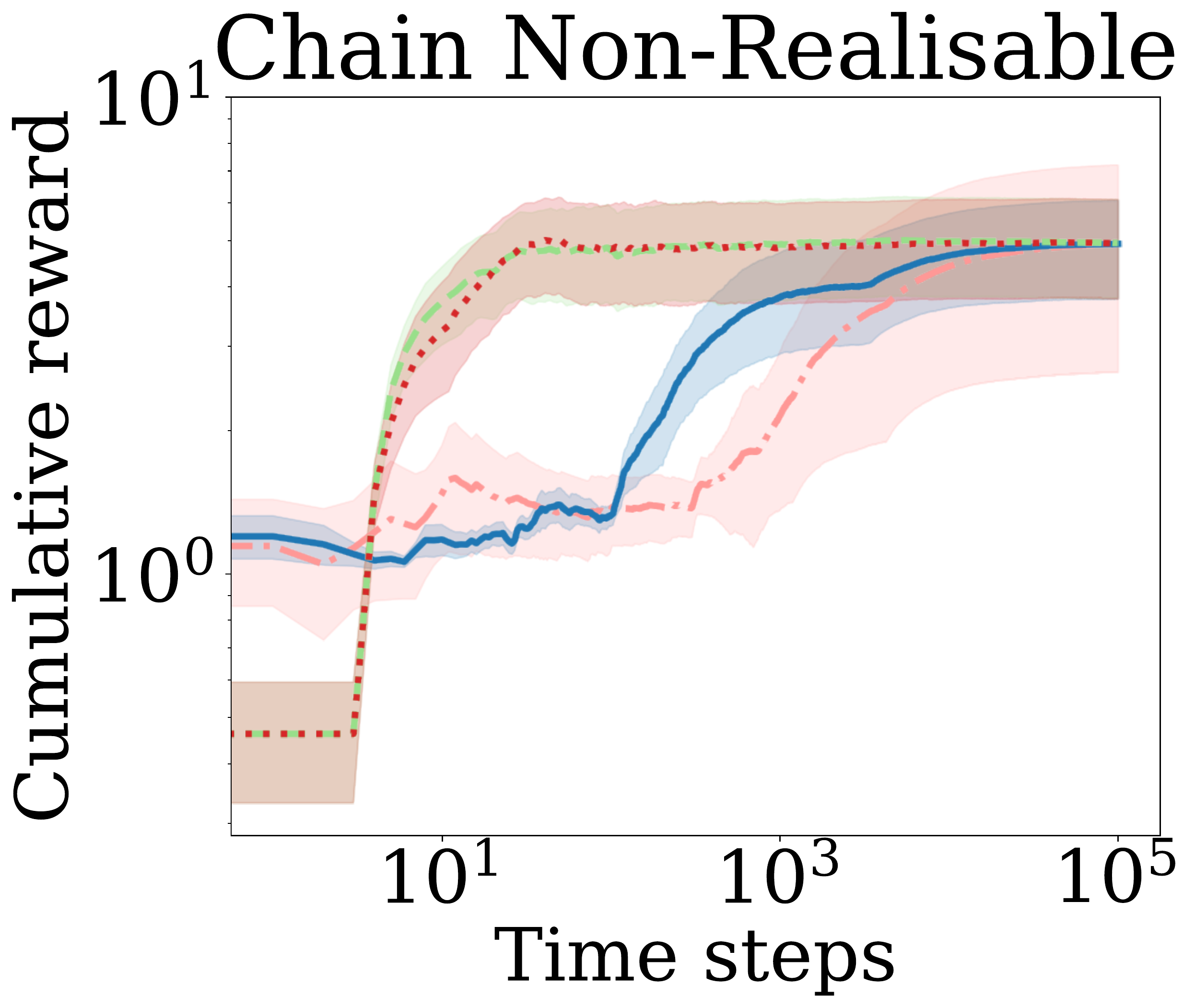}        
    \includegraphics[width=0.24\textwidth]{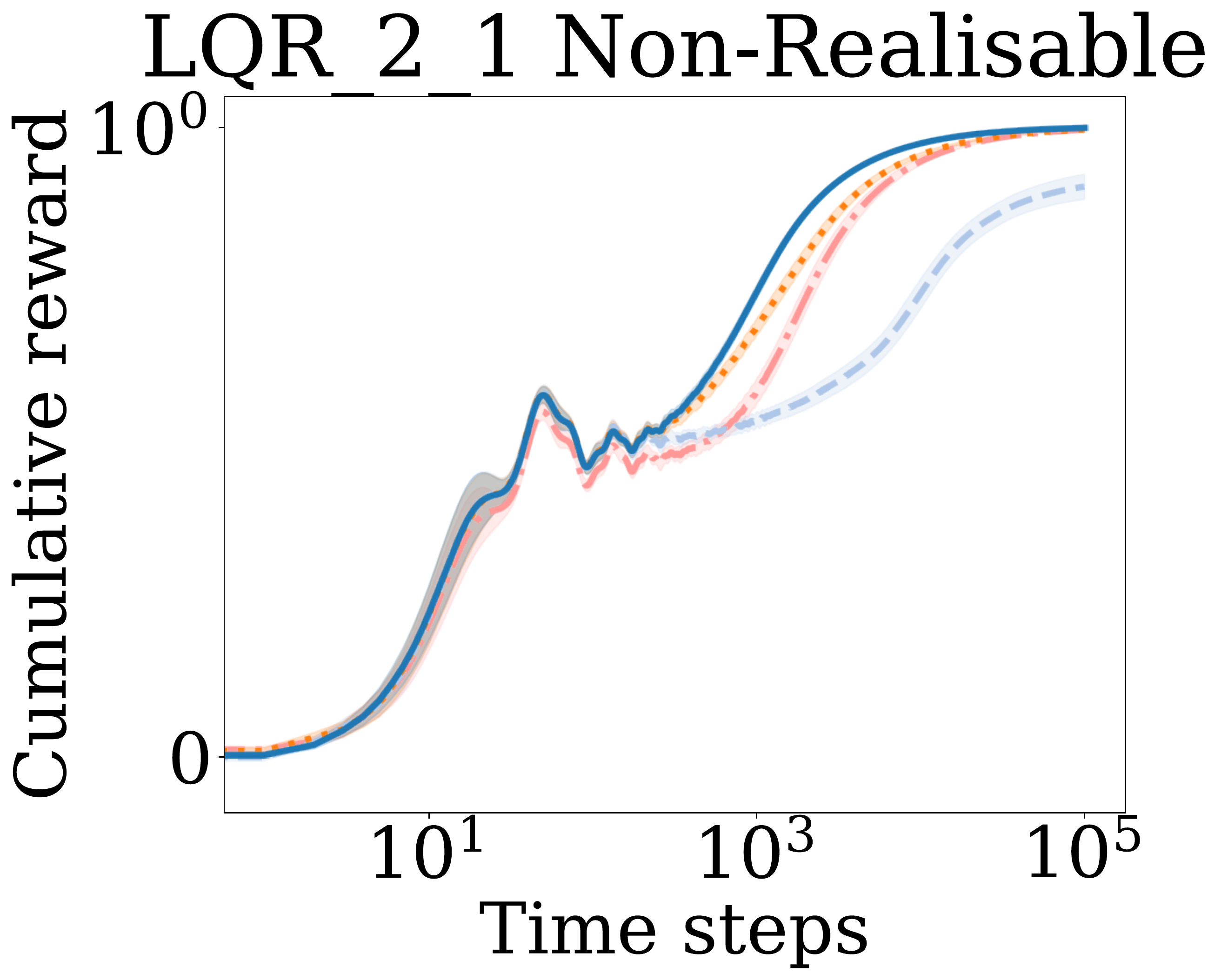} 
    \includegraphics[width=0.24\textwidth]{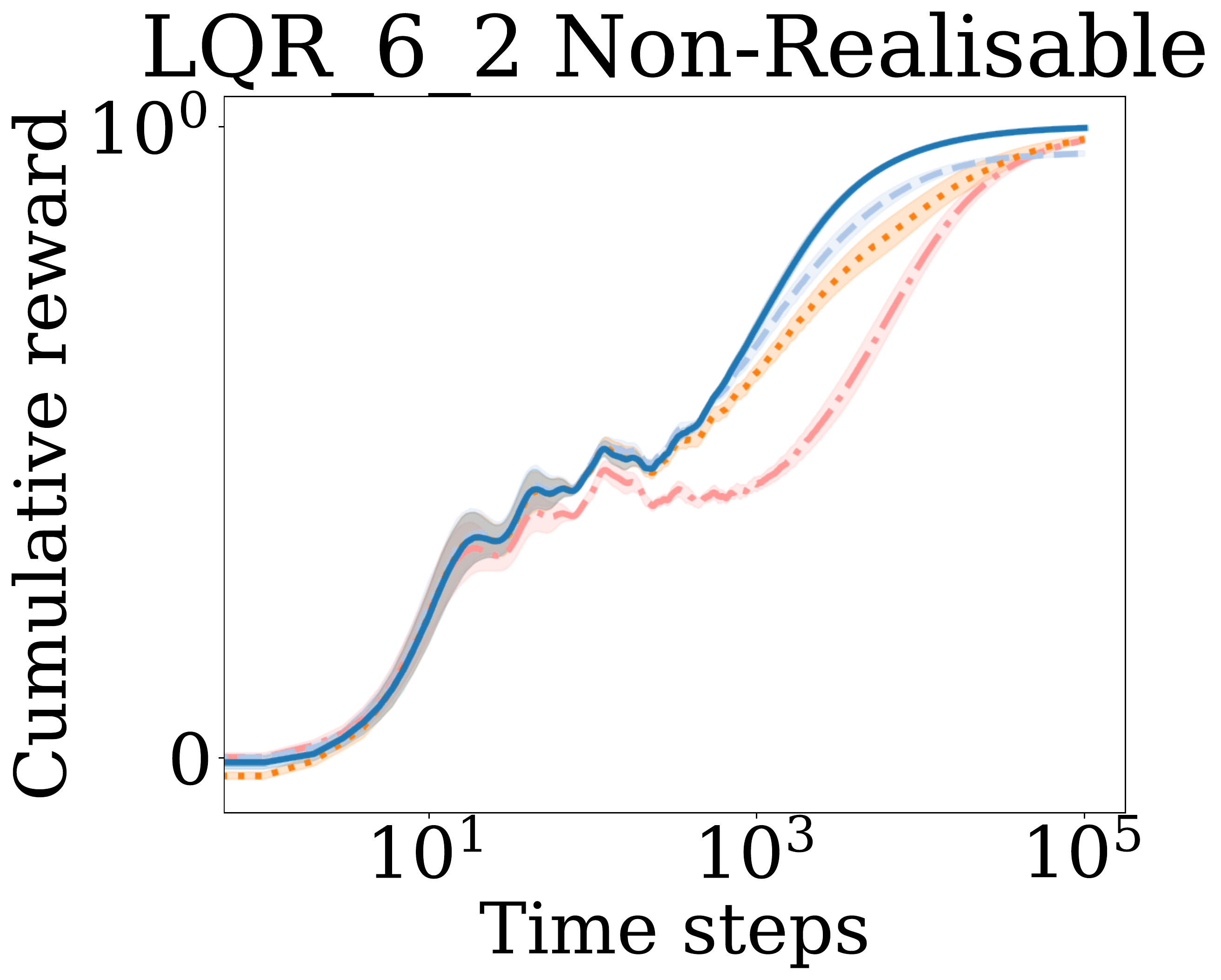}\\
    \includegraphics[width=0.74\textwidth]{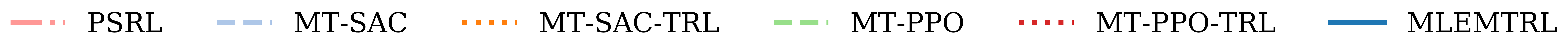}
    \caption{Depicted is the average cumulative reward at every time step computed over $10$ novel tasks in the realisable/non-realisable setting. The shaded regions represent the standard error of the average cumulative reward at the time step.}\label{fig:full_results}\vspace*{-1em}
\end{figure*}

\noindent\textbf{RL Environments.} We test the algorithms in a tabular MDP, i.e. Chain~\citep{dearden1998bayesian}, CartPole~\citep{barto1983neuronlike}, and two LQR tasks in \emph{Deepmind Control Suite}~\citep{tassa2018deepmind}: \emph{dm\_LQR\_2\_1} and \emph{dm\_LQR\_6\_2}. Further details on experimental setups are deferred to Appendix~\ref{sec:rl_env}.

\noindent\textbf{Impacts of Model Transfer with MLEMTRL.}\label{sec:impacts} We begin by evaluating the proposed algorithm in the Chain environment. The results of the said experiment are available in the leftmost column of Figure~\ref{fig:full_results}. In it, we evaluate the performance of MLEMTRL against PSRL, MT-PPO, MT-PPO-TRL. The experiments are done by varying the slippage parameter $p \in [0.00, 0.50]$ and the results are computed for each different setup of Chain from scratch. In this experiment, we can see the baseline algorithms MT-PPO and MT-PPO-TRL perform very well. This could partially be explained by PSRL and MLEMTRL not only having to learn the transition distribution but also the reward function. The value function transfer in the PPO-based baselines implicitly transfers not only the empirical transition model but also the reward function. We can see that MLEMTRL has improved learning speed compared to PSRL in both realisable and non-realisable settings. 
An additional experiment with a known reward function across tasks is shown in Figure~\ref{fig:known_reward_results}. 

In the centre and rightmost columns of Figure~\ref{fig:full_results}, we can see the results of running the algorithms in the LQR settings with the baseline algorithms PSRL, MT-SAC and MT-SAC-TRL. The variation over tasks is given by the randomness over the stiffness of the joints in the problem. In these experiments, we can see a clear advantage of MLEMTRL compared to all baselines in terms of learning speed improvements, and in some cases, asymptotic performance.

In Figure~\ref{fig:full_results}, the performance metric is the average cumulative reward at every time step, for $10^5$ time steps and the shaded region represents the standard deviation, where the statistics are computed over $10$ independent tasks.


\noindent\textbf{Impact of Realisability Gap on Regret.} Now, we further illustrate the observed relation between model dissimilarity and degradation in performance. Figure~\ref{fig:norms} depicts the regret against the KL-divergence of the target model to the best proxy model in the convex set. We observe that model dissimilarity influences the performance gap in MLEMTRL. 
This is also justified in the Section~\ref{sec:bounds} where the bounds have an explicit dependency on the model difference. In this figure, only the non-zero regret experiments are shown. This is to have an idea of which models result in poor performance. As its shown, it is those models that are very dissimilar. Additional results in Figure~\ref{fig:time} further illustrate the dependency on model similarity.

\begin{figure}[t!]
    \centering
    \includegraphics[width=0.31\textwidth]{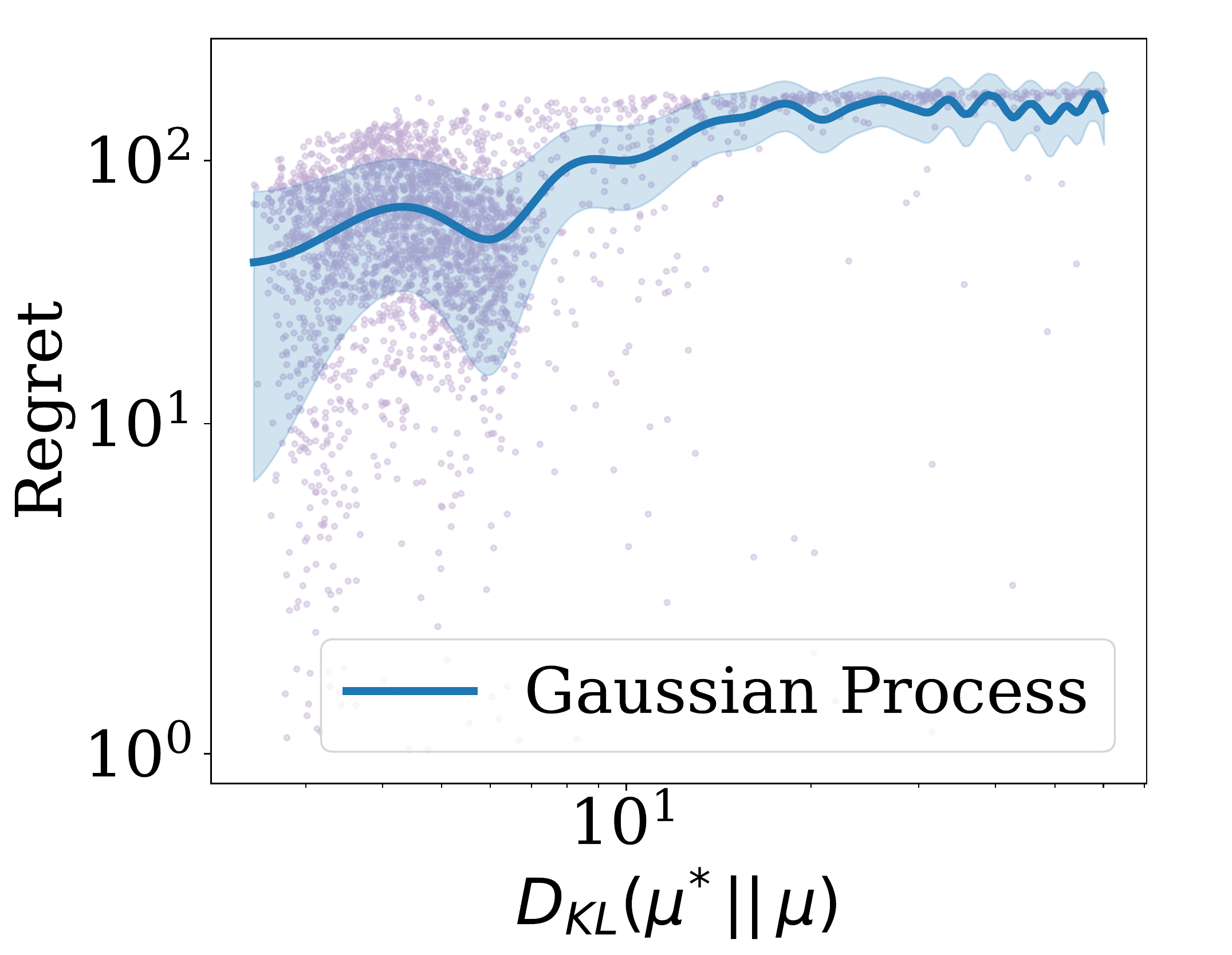}
    \caption{A log-log plot of the regret against the KL-divergence between the true MDP and the best proxy model in CartPole. The thick blue line is a Gaussian Process regression model fitted on the observed data (in purple).}\label{fig:norms}
\end{figure}

\noindent\textbf{Summary of Results.} In the experiments, we sought to identify whether the proposed algorithm shows superiority in terms of the transfer learning goals given by~\citet{langley2006transfer}. In the LQR-based environments, we can see a clear superiority in terms of learning speed compared to all baselines and in some cases, an asymptotic improvement. In the Chain environment the proposed algorithm outperforms PSRL in terms of learning speed.



Additional ablation studies showing how MLEMTRL can be augmented with a hierarchical procedure to take the empirical model into account in addition to the source models are depicted in Figures~\ref{fig:meta_results} and~\ref{fig:sac_results}. They show how the use of multi-task and transfer RL can improve performance over a standard RL approach. 


\section{Discussions and Future Work}\label{sec:discussion}
In this work, we aim to answer: 1. \emph{How can we accurately construct a model using a set of source models for an RL agent deployed in the wild?}
2. \emph{Does the constructed model allows us to perform efficient planning and yield improvements over learning from scratch?} 
Our answer to the first question is by adopting the \emph{Model Transfer Reinforcement Learning} framework and weighting existing knowledge together with data from the novel task. We accomplished this by the way of a maximum likelihood procedure, which resulted in a novel algorithm, MLEMTRL, consisting of a model identification stage and a model-based planning stage. 
The second question is answered by the empirical results in Section~\ref{sec:experiments} and the theoretical results in Section~\ref{sec:bounds}. We can clearly see the model allows for generalisation to novel tasks, given that the tasks are similar enough to the existing task. 

We motivate the use of our framework in settings where an agent is to be deployed in a new domain that is similar to existing, known, domains. We verify the quick, near-optimal performance of the algorithm in the case where the new domain is similar and we prove worst-case performance bounds of the algorithm in both the realisable and non-realisable settings.





\acks{This work was partially supported by the Wallenberg AI, Autonomous Systems and Software Program (WASP) funded by the Knut and Alice Wallenberg Foundation and the computations were performed on resources at Chalmers Centre for Computational Science and Engineering (C3SE) provided by the Swedish National Infrastructure for Computing (SNIC).}

\newpage

\vskip 0.2in
\bibliography{references}

\newpage

\appendix
\section{Detailed Proofs}

\subsection{Proof of Theorem~\ref{lemma:non-realisable}}\label{sec:proof_non_realisability}
\begin{proof}[Proof of Theorem~\ref{lemma:non-realisable}]
We begin by introducing the appropriate definitions and lemmas.  

\begin{definition}[$\epsilon$-homogeneity]\label{def:eps_homo}
    Given two MDPs $\mdp_1 = (\mathcal{S}, \mathcal{A}, \mathcal{R}, \mathcal{T}_1, \gamma)$ and $\mdp_2 = (\mathcal{S}, \mathcal{A}, \mathcal{R}, \mathcal{T}_2, \gamma)$ we say $\mdp_2$ is a $\epsilon$-homogenous partition of $\mdp_1$ with respect to $L_k$ norm if
    \begin{equation}
        \forall a \in \mathcal{A} \quad \Big(\sum_{s' \in \mathcal{S}} \big(\sum_{s\in\mathcal{S}}\mathcal{T}_1(s, a, s')-\mathcal{T}_2(s, a, s')\big)^k\Big)^{\frac{1}{k}} \leq \epsilon.
    \end{equation}
\end{definition}

Note that this definition of $\epsilon$-homogeneity is a special case of Definition 3~\citep{even2003approximate}, where the reward functions and state spaces are taken to be identical. We will only study partitions with respect to the $L_1$ norm between transition probabilities.

\begin{lemma}[Lemma 3~\citet{even2003approximate}]\label{lemma:same_pol}
    Let $\mdp_2$ be an $\epsilon$-homogenous partition of $\mdp_1$, then, with respect to $L_1$ norm, an arbitrary policy $\pol$ in $\mdp_2$ induces an $\frac{\epsilon}{(1-\gamma)^2}-$optimal policy in $\mdp_1$.
    \begin{equation}
        ||V_{\mdp_1}^\pol-V_{\mdp_2}^\pol||_\infty \leq \frac{\epsilon}{(1-\gamma)^2}.
    \end{equation}
\end{lemma}

\begin{lemma}[Lemma 4~\citet{even2003approximate}]\label{lemma:opt_opt}
    Let $\mdp_2$ be an $\epsilon$-homogenous partition of $\mdp_1$, then, with respect to $L_1$ norm, the optimal policy in $\mdp_2$ induces an $\frac{2\epsilon}{(1-\gamma)^2}-$optimal policy in $\mdp_1$. 
    \begin{equation}
        ||V_{\mdp_1}^*-V_{\mdp_2}^*||_\infty \leq \frac{2\epsilon}{(1-\gamma)^2}.
    \end{equation}
\end{lemma}

Given the assumptions in Theorem~\ref{lemma:non-realisable} hold true. Then, using the $\epsilon-$homogeneity definition in Definition~\ref{def:eps_homo}, let $\hat{\mdp}$ be an $\epsilon_{\mathrm{Estim}}-$homogenous partition of $\mdp$ and $\mdp$ be an $\epsilon_{\mathrm{Realise}}-$homogenous partition of $\mdp^*$. Under the $L_1$ norm then, we have     
    \begin{align}
        \forall a \in \mathcal{A} \quad &\Big(\sum_{s' \in \mathcal{S}} \sum_{s\in\mathcal{S}}\mathcal{T}(s, a, s')-\hat{\mathcal{T}}(s, a, s')\Big) \leq \epsilon_{\mathrm{Estim}},\\
         &\Big(\sum_{s' \in \mathcal{S}} \sum_{s\in\mathcal{S}}\mathcal{T}^*(s, a, s')-\mathcal{T}(s, a, s')\Big) \leq \epsilon_{\mathrm{Realise}}.
    \end{align}
Using triangle inequalities we can then bound the $L_1$ norm between the true underlying MDP and the maximum likelihood estimator,
\begin{align}
     \forall a \in \mathcal{A} \quad &\Big(\sum_{s' \in \mathcal{S}} \sum_{s\in\mathcal{S}}\mathcal{T}^*(s, a, s')-\hat{\mathcal{T}}(s, a, s')\Big)\\
    = \, &\Big(\sum_{s' \in \mathcal{S}} \sum_{s\in\mathcal{S}}\mathcal{T}^*(s, a, s')-\mathcal{T}(s,a,s')+\mathcal{T}(s,a,s')-\hat{\mathcal{T}}(s, a, s')\Big)\\
    \leq \, &\Big(\sum_{s' \in \mathcal{S}} \sum_{s\in\mathcal{S}}\mathcal{T}^*(s, a, s')-\mathcal{T}(s,a,s')\Big) +\Big(\sum_{s' \in \mathcal{S}} \sum_{s\in\mathcal{S}}\mathcal{T}(s, a, s')-\hat{\mathcal{T}}(s,a,s')\Big)\\ \leq \, &\epsilon_{\mathrm{Estim}}+\epsilon_{\mathrm{Realise}}.
\end{align}

Thus, $\hat{\mdp}$ is a $(\epsilon_{\mathrm{Estim}}+\epsilon_{\mathrm{Realise}})-$homogenous partition of $\mdp^*$. The next steps involves creating a bound on the performance gap between the value functions and policies in $||V_{{\mdp^*}}^{*}-V_{{\mdp^*}}^{\hat{\pol}}||_\infty$. Using triangle inequalities the performance gap can be expanded further,

\begin{equation}
\begin{aligned}\label{eq:triangle}
    ||V_{{\mdp^*}}^{*}-V_{{\mdp^*}}^{\hat{\pol}}||_\infty &= ||V_{{\mdp^*}}^{*}-V_{\hat{\mdp}}^{\hat{\pol}}+V_{\hat{\mdp}}^{\hat{\pol}}-V_{{\mdp^*}}^{\hat{\pol}}||_\infty\\
    &\leq ||V_{{\mdp^*}}^{*}-V_{\hat{\mdp}}^{\hat{\pol}}||_\infty+||V_{\hat{\mdp}}^{\hat{\pol}}-V_{{\mdp^*}}^{\hat{\pol}}||_\infty.
\end{aligned}
\end{equation}

The first term on the right side of the inequality in Eq.~\ref{eq:triangle} can be bounded using Lemma~\ref{lemma:opt_opt} since $\hat{\pol}$ is the optimal policy in $\hat{\mdp}$,

\begin{equation}
    ||V^{*}_{\mdp^{*}}-V_{\hat{\mdp}}^{\hat{\pol}}||_\infty \leq \frac{2(\epsilon_{\mathrm{Estim}}+\epsilon_{\mathrm{Realise}})}{(1-\gamma)^2}.
\end{equation}
Likewise, the second term in the inequality can be bounded using Lemma~\ref{lemma:same_pol},

\begin{equation}
    ||V_{\hat{\mdp}}^{\hat{\pol}}-V_{\mdp^*}^{\hat{\pol}}||_\infty \leq \frac{\epsilon_{\mathrm{Estim}}+\epsilon_{\mathrm{Realise}}}{(1-\gamma)^2}.
\end{equation}

Combining these two terms yields us $||V_{{\mdp^*}}^{*}-V_{{\mdp^*}}^{\hat{\pol}}||_\infty \leq \frac{3(\epsilon_{\mathrm{Estim}}+\epsilon_{\mathrm{Realise}})}{(1-\gamma)^2}$.
\end{proof}

\subsection{Proof of Remark~\ref{remark:l1_norm}}\label{sec:proof_remark}
\begin{proof}[Proof of Remark~\ref{remark:l1_norm}]
    The analysis of the concentration of $\epsilon_{\mathrm{Estim}}$ follows the works of~\citet{auer2008near, qian2020concentration}. Let $\Delta^{\mathcal{S}}$ be the $(\mathcal{S}-1)-$dimensional simplex and $\mathcal{T}_{s,a}^* \in \Delta^{\mathcal{S}}$ be the transition kernel for a state-action pair of the true underlying MDP $\mu^*$ and $\hat{\mathcal{T}}_{s,a} \in \Delta^{\mathcal{S}}$ a random vector. If $\hat{\mathcal{T}}_{s,a}$ is taken to be the empirical estimate of $\mathcal{T}_{s,a}^*$ then the following lemma can be invoked.
    
    \begin{proposition}[\citet{weissman2003inequalities}]\label{lemma:weissman}
        Let $\mathcal{T}_{s,a}^* \in \Delta^{\mathcal{S}}$ and $\hat{\mathcal{T}}_{s,a} \sim \frac{1}{n^{s,a}}\mathrm{Multinomial}(n^{s,a}, \mathcal{T}_{s,a}^*)$. Then, for $S \geq 2, \delta \in [0, 1]$ and $n^{s,a} \geq 1$,
        \begin{equation}
            \mathbb{P}\Bigg(||\mathcal{T}_{s,a}^*-\hat{\mathcal{T}}_{s,a}||_1\geq \sqrt{\frac{2S\log(2/\delta)}{n^{s,a}}}\Bigg)\leq \mathbb{P}\Bigg(||\mathcal{T}_{s,a}^*-\hat{\mathcal{T}}_{s,a}||_1\geq \sqrt{\frac{2\log\big((2^S -2)/\delta)\big)}{n^{s,a}}}\Bigg)\leq \delta.
        \end{equation}
    \end{proposition}
    The invocation of Proposition~\ref{lemma:weissman}, with $T=\sum_{s\in\mathcal{S}}\sum_{a\in\mathcal{A}}n^{s,a}$ yields us a $L_1$ norm bound for the difference of transition kernels associated with a particular state-action pair $s, a$. Next, union bounding over all possible state and action combinations yields us a bound on the total $L_1$ norm.
    \begin{align}
        &\mathbb{P}\Bigg(\bigcup_{s\in\mathcal{S}}\bigcup_{a\in\mathcal{A}}\Big(||\mathcal{T}_{s,a}^*-\hat{\mathcal{T}}_{s,a}||_1\geq \sqrt{\frac{2\log\big((2^S -2)/\delta)\big)}{n^{s,a}}}\Big)\Bigg)\\
        \leq &\sum_{s\in\mathcal{S}}\sum_{a\in\mathcal{A}} \mathbb{P}\Bigg(||\mathcal{T}_{s,a}^*-\hat{\mathcal{T}}_{s,a}||_1\geq \sqrt{\frac{2\log\big((2^S -2)/\delta)\big)}{n^{s,a}}}\Bigg)\\
        \leq &SA\delta.
    \end{align}
    From this, we have that, with probability $1-SA\delta$,
    \begin{equation}
        ||\mathcal{T}^*-\hat{\mathcal{T}}||_1 \leq \epsilon_{\mathrm{Estim}} \leq \sum_{s\in\mathcal{S}}\sum_{a\in\mathcal{A}}\sqrt{\frac{2\log\big((2^S -2)/\delta)\big)}{n^{s,a}}}.
    \end{equation}
    The total $L_1$ norm then scales on the order of $\mathcal{O}(SA\sqrt{S-\log(\delta)}/\sqrt{T})$, which is the final result.
\end{proof}

\clearpage
\section{Details of Planning: \textsc{RiccatiIteration}}\label{sec:riccati}
An LQR-based control system is defined by its system matrices~\citep{kalman1960new}. Let $d_s$ be the state dimensionality and $d_a$ be the action dimensionality. Then, $\mathbf{A} \in \mathbb{R}^{d_s}\times\mathbb{R}^{d_s}$ is a matrix describing state associated state transitions. $\mathbf{B} \in \mathbb{R}^{d_s}\times\mathbb{R}^{d_a}$ is a matrix describing control associated state transitions. The final two system matrices are cost related with $\mathbf{Q} \in \mathbb{R}^{d_s}\times\mathbb{R}^{d_s}$ being a positive definite cost matrix of states and $\mathbf{R} \in \mathbb{R}^{d_a}\times\mathbb{R}^{d_a}$ a positive definite cost matrix of control inputs. The transition model described under this model is given by,
\begin{equation}
   \bm{s}_{t+1}-\bm{s}_t = \mathbf{A}\bm{s}_t + \mathbf{B}\bm{a}_t.
\end{equation}

When an MDP is mentioned in the context of an LQR system in this work, the MDP is the set of system matrices. Further, the cost (or reward) of a policy $\pol$ under an MDP $\mdp$ is
\begin{equation}
    V_{\mdp}^\pol = \sum_{t=0}^T \bm{s}_t^\top \mathbf{Q}\bm{s}_t + \bm{a}_t^\top \mathbf{R}\bm{a}_t.
\end{equation}
Optimal policy identification can be accomplished using~\citet{willems1971least}. It begins by solving for the cost-to-go matrix $\mathbf{P}$ by,
\begin{align*}
   \underset{\mathbf{P}}{\textrm{solve}}~~~~\mathbf{A}^\top \mathbf{P}\mathbf{A}-\mathbf{P}+\mathbf{Q}-(\mathbf{A}^\top \mathbf{P}\mathbf{B})(\mathbf{R}+\mathbf{B}^\top\mathbf{P}\mathbf{B})^{-1}(\mathbf{B}^\top \mathbf{P}\mathbf{A}) = 0.
\end{align*}

Then, using $\mathbf{P}$ the control input $\bm{a}$ for a particular state $\bm{s}$ is
\begin{equation}
   \bm{a} = -(\mathbf{R}+\mathbf{B}^\top\mathbf{P}\mathbf{B})^{-1}(\mathbf{B}^\top\mathbf{P}\mathbf{A})\bm{s}.
\end{equation}

With some abuse of notation and for compactness, we allow ourselves to write $\bm{a}_t = -(\mathbf{R}+\mathbf{B}^\top\mathbf{P}\mathbf{B})^{-1}(\mathbf{B}^\top\mathbf{P}\mathbf{A})\bm{s}_t$ for $\bm{a}_t \sim \pol(\bm{s}_t)$.

\section{Meta-Algorithm for MLEMTRL in the Non-Realisable Setting}\label{sec:meta_mlemtrl}

In order to guarantee good performance even in the non-realisable setting one might think of adding the target task to the set of source tasks or constructing a meta-algorithm, combining the model estimated by MLEMTRL and the empirical estimation of the target task. In this section we propose a meta-algorithm based on the latter, in Algorithm~\ref{alg:meta-mlemtrl}. The main change in the algorithm is internally keeping track of the empirical model and on Line~\ref{lin:meta-mlemtrl}, computing a posterior probability distribution over the respective models by weighting the two likelihoods together with their respective priors. How much the meta-algorithm should focus on the empirical model is then decided by the prior, because $\ell_{\mathrm{Empirical}} \geq \ell_{\mathrm{MLEM}}$. For experimental results using this algorithm, see Figure~\ref{fig:meta_results}.

\begin{algorithm}[ht!]
\caption{Meta-MLEMTRL}\label{alg:meta-mlemtrl}
\begin{algorithmic}[1]
\STATE \textbf{Input:} prior $p$, weights $\bm{w}^0$, $m$ source MDPs $\mathcal{M}_s$, data $D_0$, discount factor $\gamma$, iterations $T$.
\FOR{$t=0, \hdots, T$}
\STATE\textsc{// Stage 1: Obtain Model Weights //}
\STATE $\bm{w}^{t+1}\leftarrow  \textsc{MLEMTRL}(\bm{w}^t, \mathcal{M}_s, \mathcal{D}_t, \gamma, 1)$
\STATE Estimate the MDP: $\mdp^{t+1} = \sum_{i=1}^m w_i \mdp_i$
\STATE Compute log-likelihood $\ell_{\mathrm{MLEM}}^{t+1} = \log \mathbb{P}(\mathcal{D}_t \, | \, \mdp^{t+1})$
\STATE Compute log-likelihood of empirical model  $\ell_{\mathrm{Empirical}}^{t+1} = \log \mathbb{P}(\mathcal{D}_t \, | \, \hat{\mdp}^{t+1})$ 
\STATE Sample $\tilde{\mdp}^{t+1}$ as $\mdp^{t+1}$ w.p. $\propto p\exp\Big(\ell_{\mathrm{MLEM}}^{t+1}\Big)$ and $\hat{\mdp}^{t+1}$ w.p. $\propto (1-p)\exp\Big(\ell_{\mathrm{Empirical}}^{t+1}\Big)$.\label{lin:meta-mlemtrl}
\STATE\textsc{// Stage 2: Model-based Planning //}
\STATE Compute the policy: $\pol^{t+1} \in \underset{\pol}{\arg\max} \, V_{\tilde{\mdp}^{t+1}}^\pol$
\STATE\textsc{// Control //}
\STATE Observe $s_{t+1}, r_{t+1} \sim \mdp^{*}(s_t, a_t), a_t\sim \pol^{t+1}(s_t)$
\STATE Update the dataset $D_{t+1} = D_t \cup \{s_t, a_t, s_{t+1}, r_{t+1}\}$
\ENDFOR
\STATE \textbf{return} An estimated MDP model $\tilde{\mdp}^T$ and a policy $\pol^T$
\end{algorithmic}
\end{algorithm}

\begin{figure}[h!]
    \centering
    \includegraphics[width=0.48\textwidth]{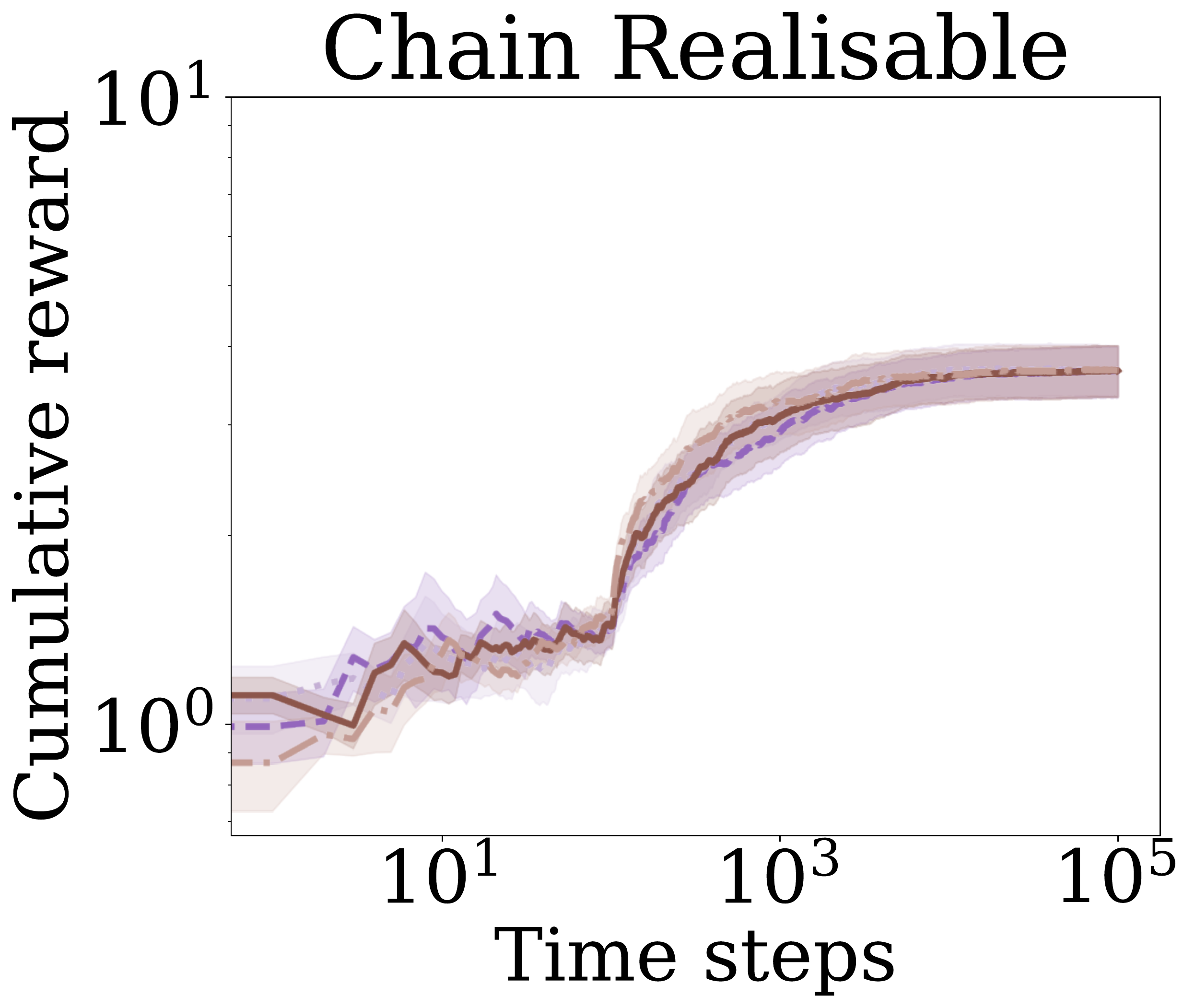} 
    \includegraphics[width=0.48\textwidth]{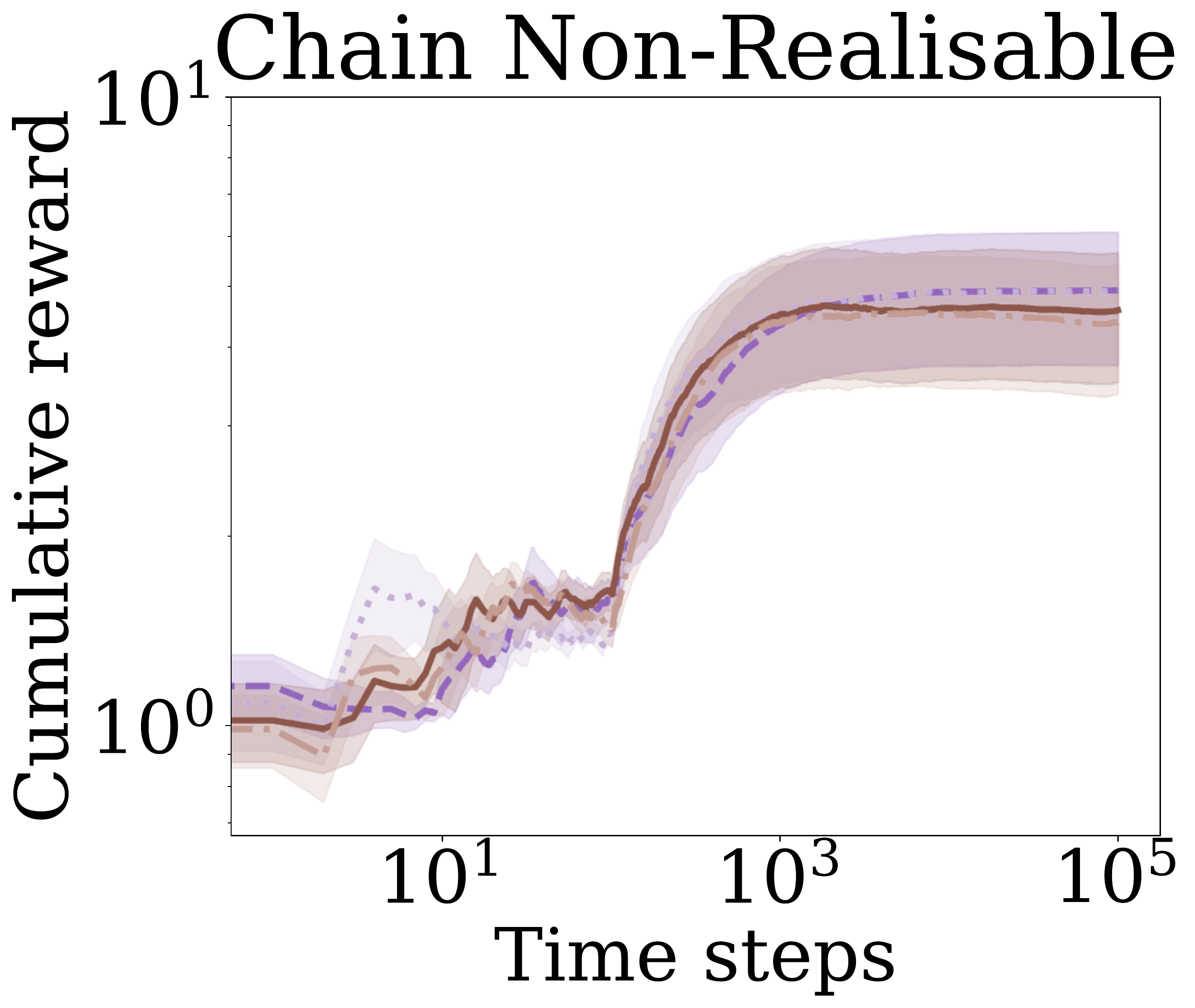}\\
    \includegraphics[width=0.96\textwidth]{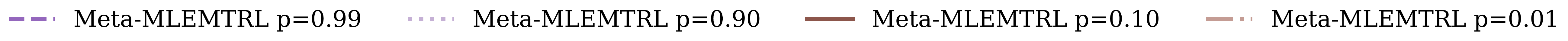}
    \caption{Figure depicting an ablation study of the prior parameter $p$ in the Meta-MLEMTRL algorithm. The y-axis is the average cumulative reward at each time step computed over $10$ novel tasks and the shaded region represents the standard error. When $p=1$, the algorithm reduces to MLEMTRL and when $p=0$ the algorithm reduces to standard maximum likelihood model estimation.}
    \label{fig:meta_results}
\end{figure}

\section{Additional Experimental Analysis}

\subsection{Experimental Setup}\label{sec:rl_env}
The experiments are deployed in \textsc{Python 3.7}, with support from \textsc{SciPy}~\citep{virtanen2020scipy}, \textsc{Stable-baselines3}~\citep{raffin2021stable} and ran on a i5-4690k CPU and a GTX-960 GPU. The parameters for the variations of SAC and PPO are kept to be the default ones.

\noindent\textbf{RL Environments: Chain.} A common testbed for RL algorithms in tabular settings is the Chain~\citep{dearden1998bayesian} environment. In it, there is a chain of states where the agent can either walk forward or backward. At the end of the chain, there is a state yielding the highest rewards. At every step, there is a chance of the effect of the opposite action occurring. This is denoted as the slipping probability. The slippage parameter is also what is used to create the source models, in this case, those parameters are $\{0.01, 0.20, 0.50\}$. 
For PSRL and MLEMTRL we use a product-NormalGamma prior over the reward functions. For PSRL, we use product-Dirichlet priors over the transition matrix.

\noindent\textbf{RL Environments: LQR Tasks.}
We investigate two LQR tasks in the \emph{Deepmind Control Suite}~\citet{tassa2018deepmind}, namely \emph{dm\_LQR\_2\_1} and \emph{dm\_LQR\_6\_2}. These environments are continuous state and actions whereby the task is to control a two joint one actuator and six joint two actuators towards the center of the platform for the two tasks, respectively. They consist of unbounded control inputs and rewards with the state spaces $s \in \mathbb{R}^4$ and $s\in\mathbb{R}^{12}$, respectively. In the Deepmind Control suite every task is made to be different by varying the seed at creation. The seed determines the stiffness of the joints.

\noindent\textbf{RL Environments: CartPole.} We also conduct some experiments on the CartPole~\citet{barto1983neuronlike} environment. In this case, we use a continuous control version of it and formulate it as a LQR problem. The environment has a single continuous action and a state space $s \in \mathbb{R}^4$. To create different tasks we vary the environmental parameters of the problem, namely the gravity, mass of cart, mass of pole the length of the pole.

\subsection{Impacts of Realisability}
In the experiment depicted in Figure~\ref{fig:time}, we investigate the convergence rate and the jumpstart improvement of the MLEMTRL algorithm on $100$ independent target MDP realisations at six different levels of divergence. The divergence is measured from the centroid of the convex hull to the target MDP. Further, in the topmost row, all of the target MDPs belong to the convex hull of source models.

\begin{figure}[h!]
   \centering
   \includegraphics[width=0.55\textwidth]{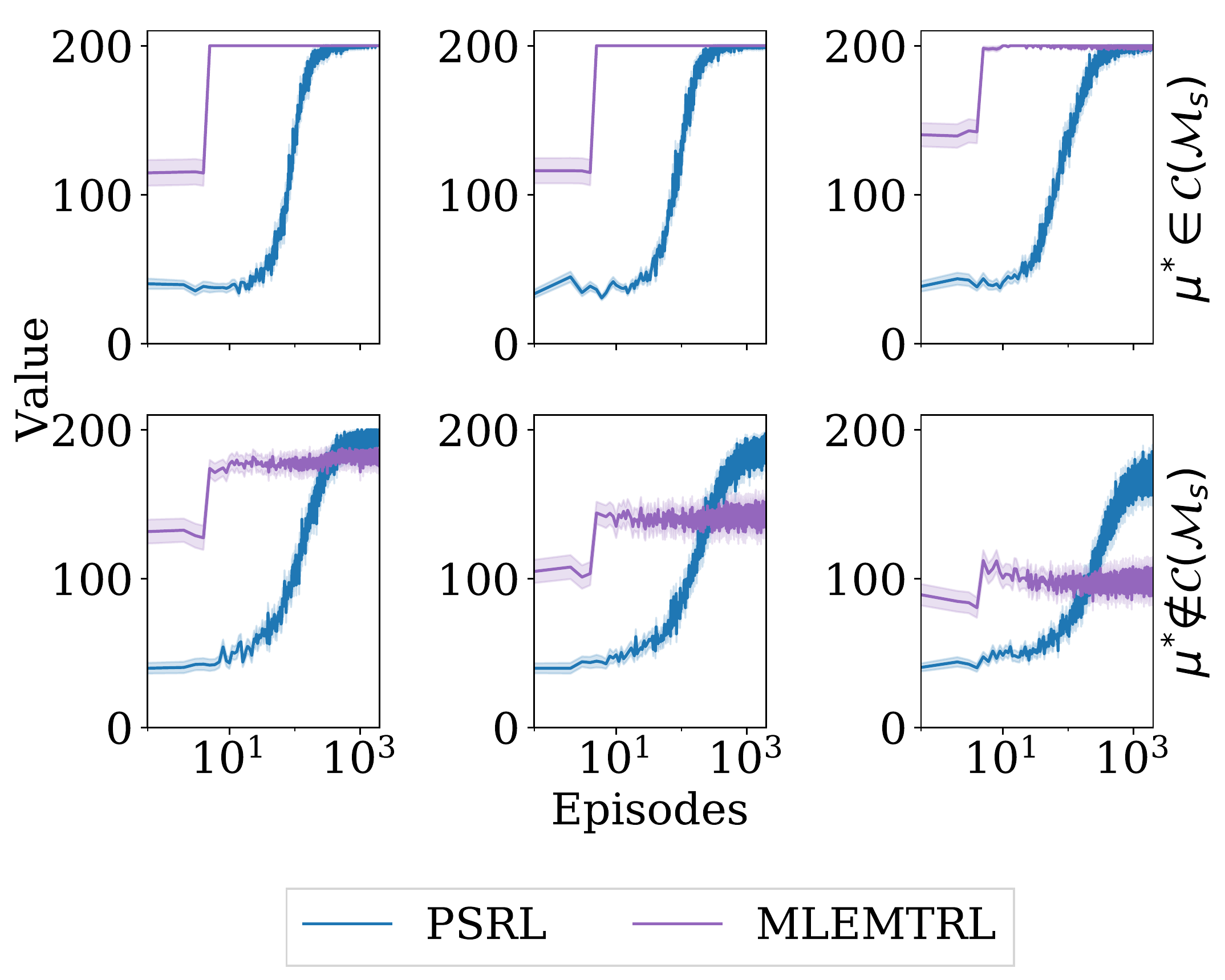}
   \caption{We compare MLEMTRL against PSRL in terms of convergence speed and early performance. The value functions of the algorithms are plotted against the log of the number of episodes. The shaded area is the $68\%$ confidence of the mean over multiple MDPs with the same model dissimilarity. In the topmost row, all of the true MDPs are within the convex hull $\mathcal{C}(\mathcal{M}_s)$. In the bottom row, the MDPs are outside. As you go from top-left to bottom-right, the divergence from the true model to the average model in the convex hull increases. For utility, higher values are better.}  \label{fig:time}
\end{figure}
As we can see, in this setting, identification of the true model occurs rapidly. One reason for this is because of the near-determinism of the environment. Compared to the agent learning from scratch, we observe zero-regret with faster convergence. As we go from top-left to bottom-right, the divergence increases. For the bottom-most row, we can again observe a faster learning rate. In this case, the degradation in performance increases with the divergence, resulting in poor performance in the final case. The experiment demonstrates that under the TRL framework, we require that the source models are not too dissimilar from the target model.

\subsection{Impacts of Multi-Task Learning as a Baseline}
In Figure~\ref{fig:sac_results} we investigate the performance increase of using multi-task RL and transfer RL compared to regular RL. The baseline algorithm is Soft Actor-Critic and its associated multi-task and transfer learning formulations. As we can see, \textbf{MT-SAC-TRL} appears to have overall strongest performance, with \textbf{MT-SAC} a close second. Because of the nature of the problem (unbounded negative rewards), it is also possible for the algorithms to diverge during learning, which further strengthens the argument for using multi-task or transfer learning for robustness.

\begin{figure}[h!]
    \centering
    \includegraphics[width=0.48\textwidth]{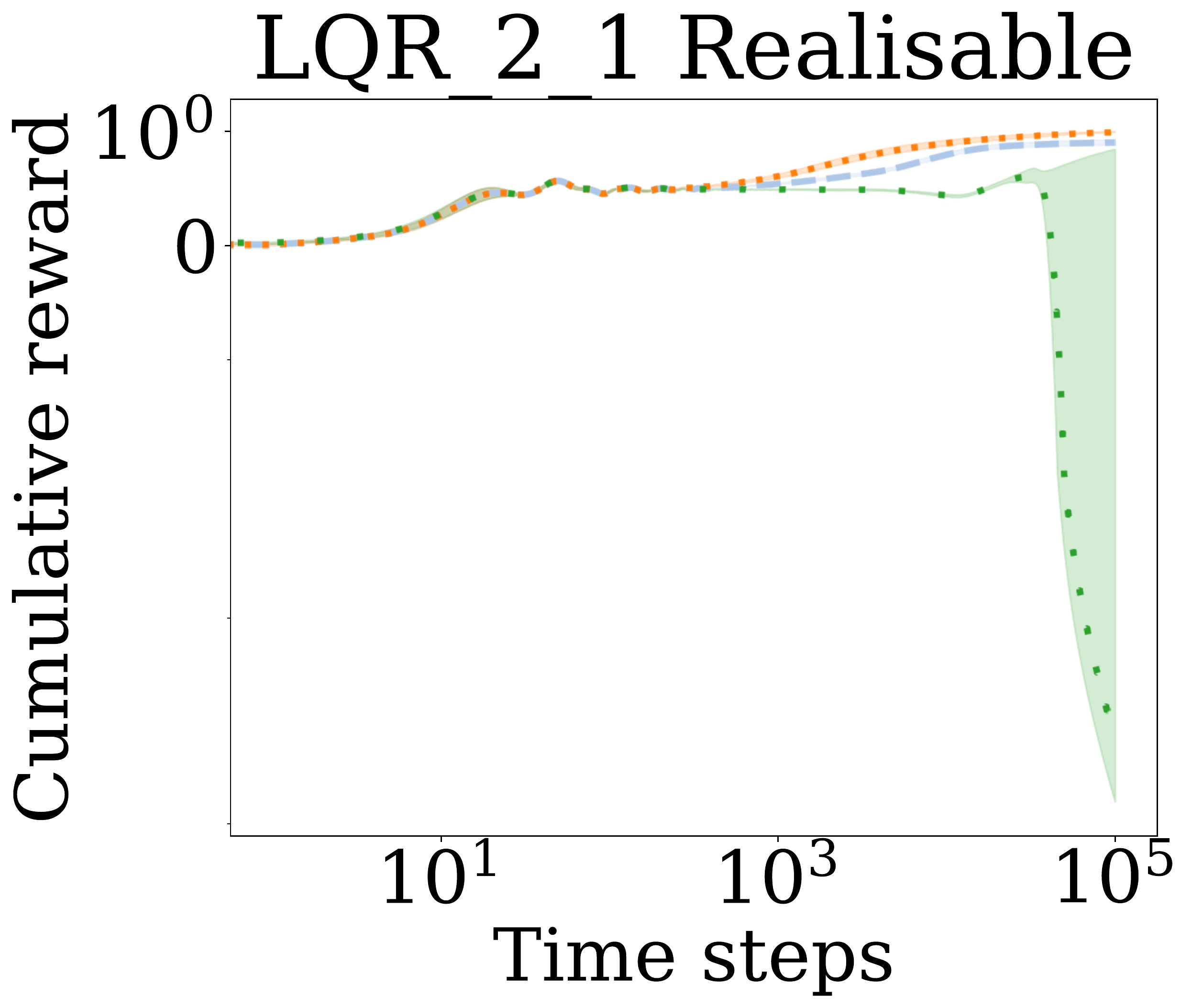} 
    \includegraphics[width=0.48\textwidth]{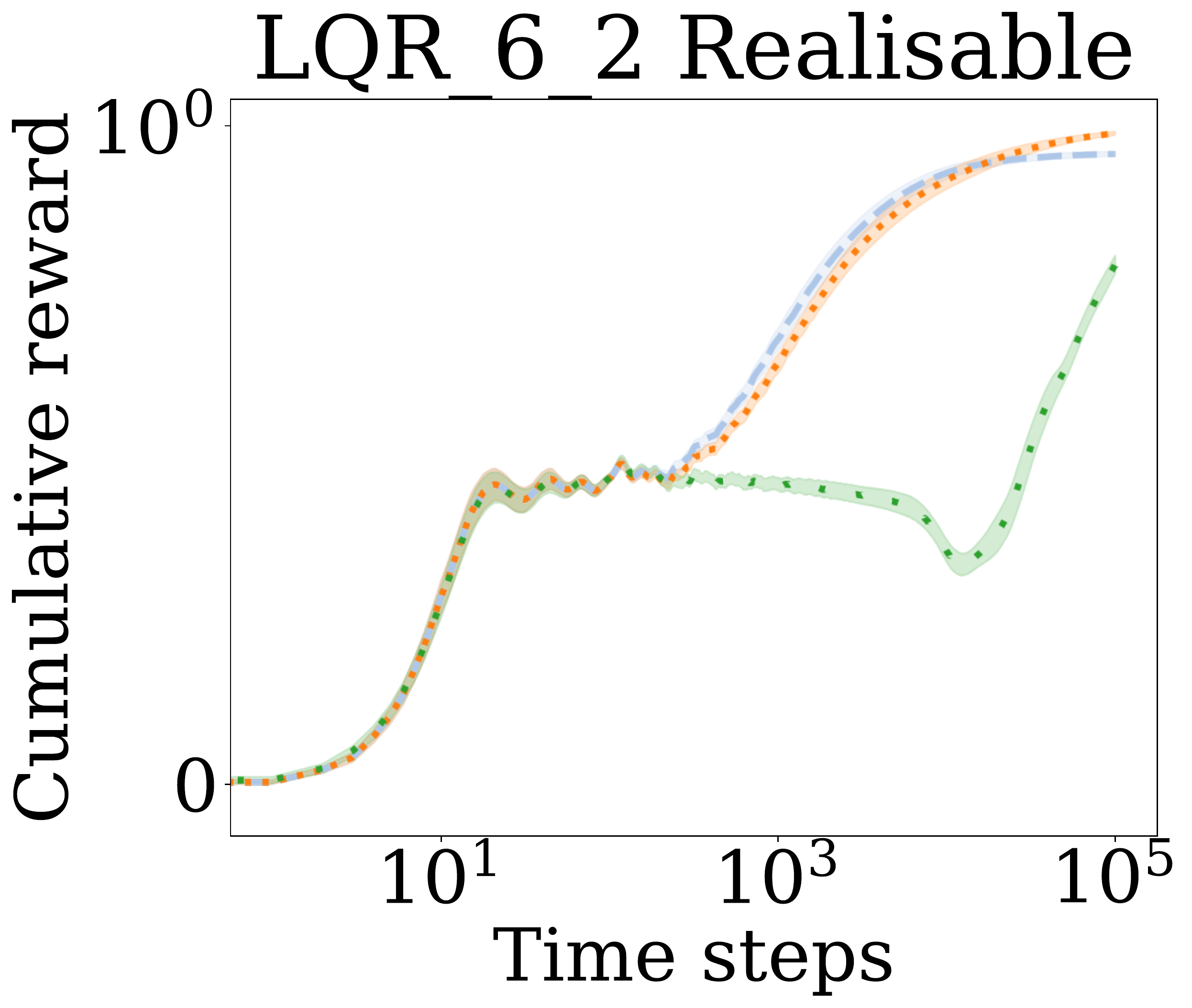}\\
    \includegraphics[width=0.48\textwidth]{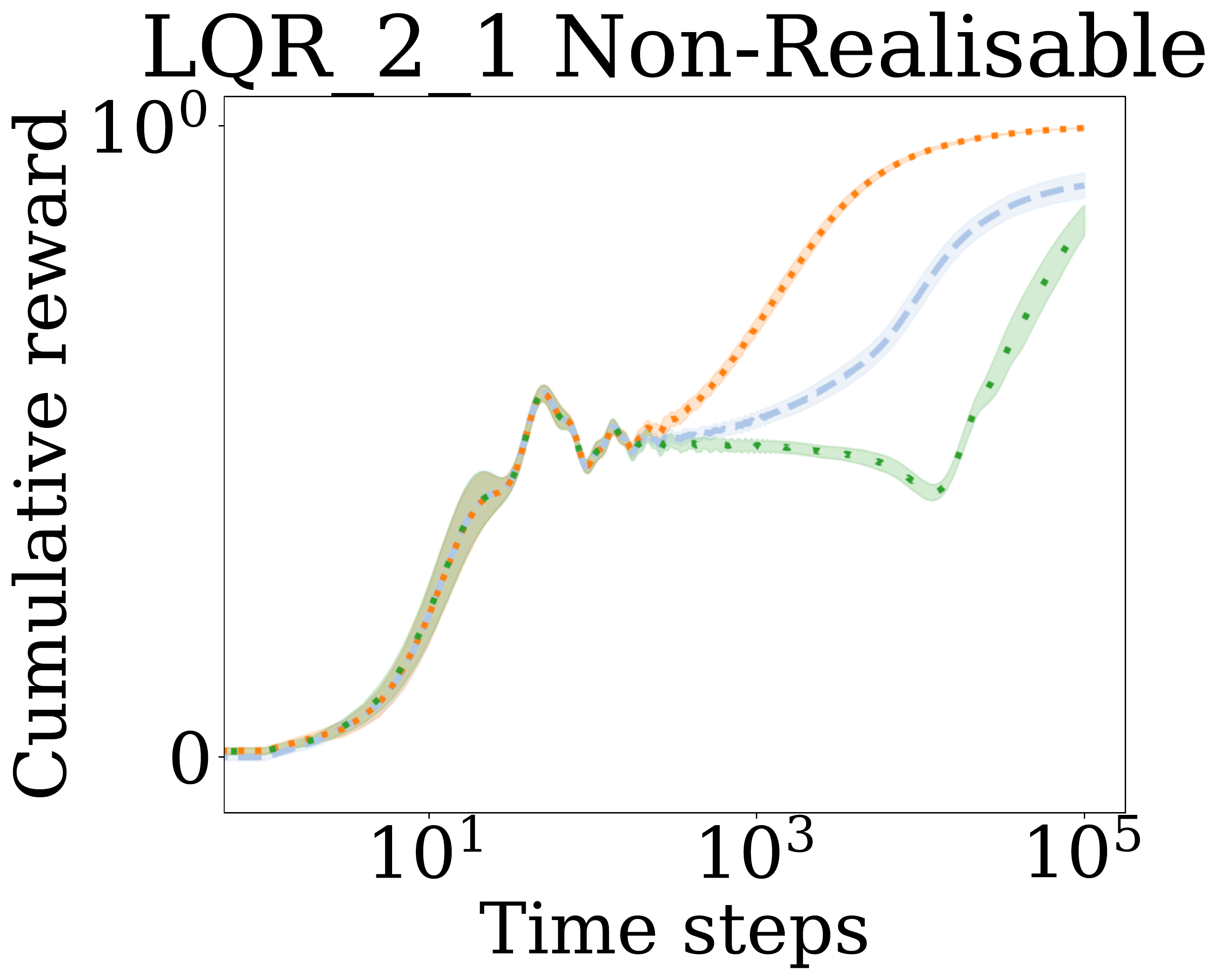} 
    \includegraphics[width=0.48\textwidth]{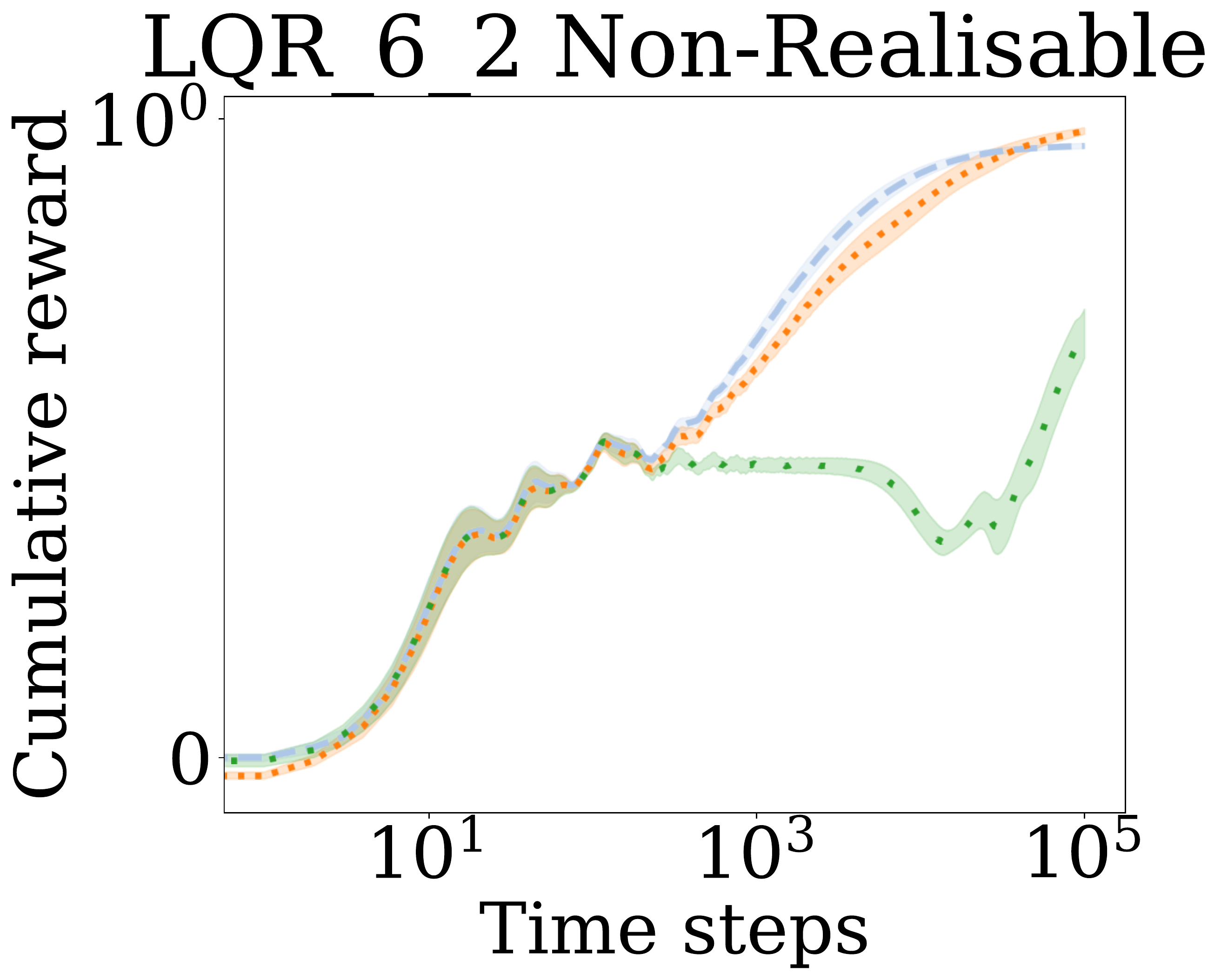}\\
    \includegraphics[width=0.96\textwidth]{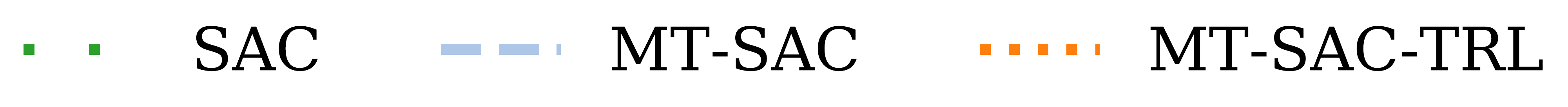}
    \caption{Figure depicting the performance boost of using multi-task and transfer reinforcement learning compared to standard reinforcement learning. The y-axis represents average cumulative reward at every time step and the shaded region is the standard error.}
    \label{fig:sac_results}
\end{figure}

\newpage
\subsection{Model-based Transfer Reinforcement Learning with Known Reward Function}

In Figure~\ref{fig:known_reward_results}, we aim to contrast the difference from the figure in the main paper where now the reward function is known a priori to \textbf{MLEMTRL} and \textbf{PSRL}.
\begin{figure}[h!]
    \centering
    \includegraphics[width=0.48\textwidth]{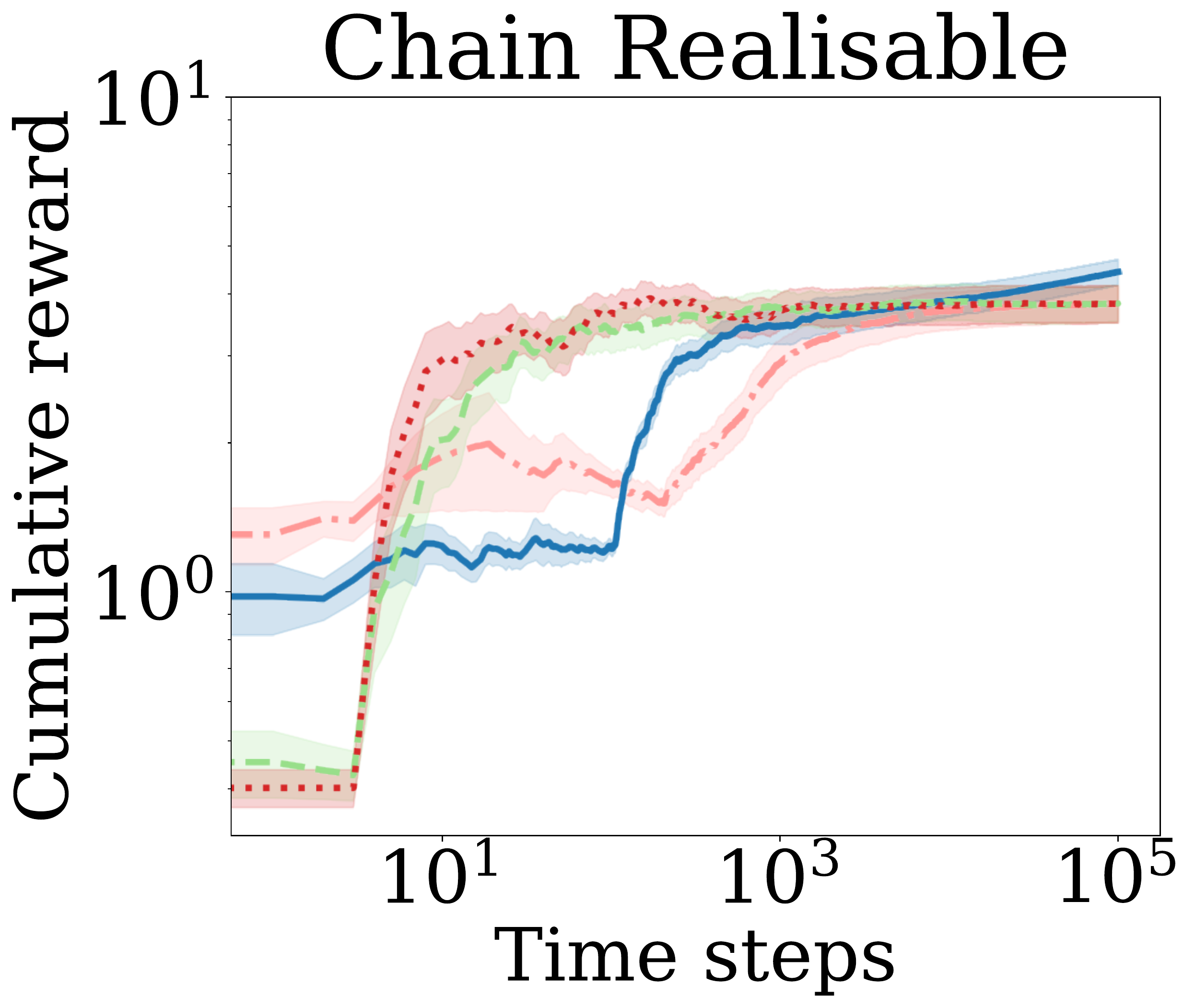} 
    \includegraphics[width=0.48\textwidth]{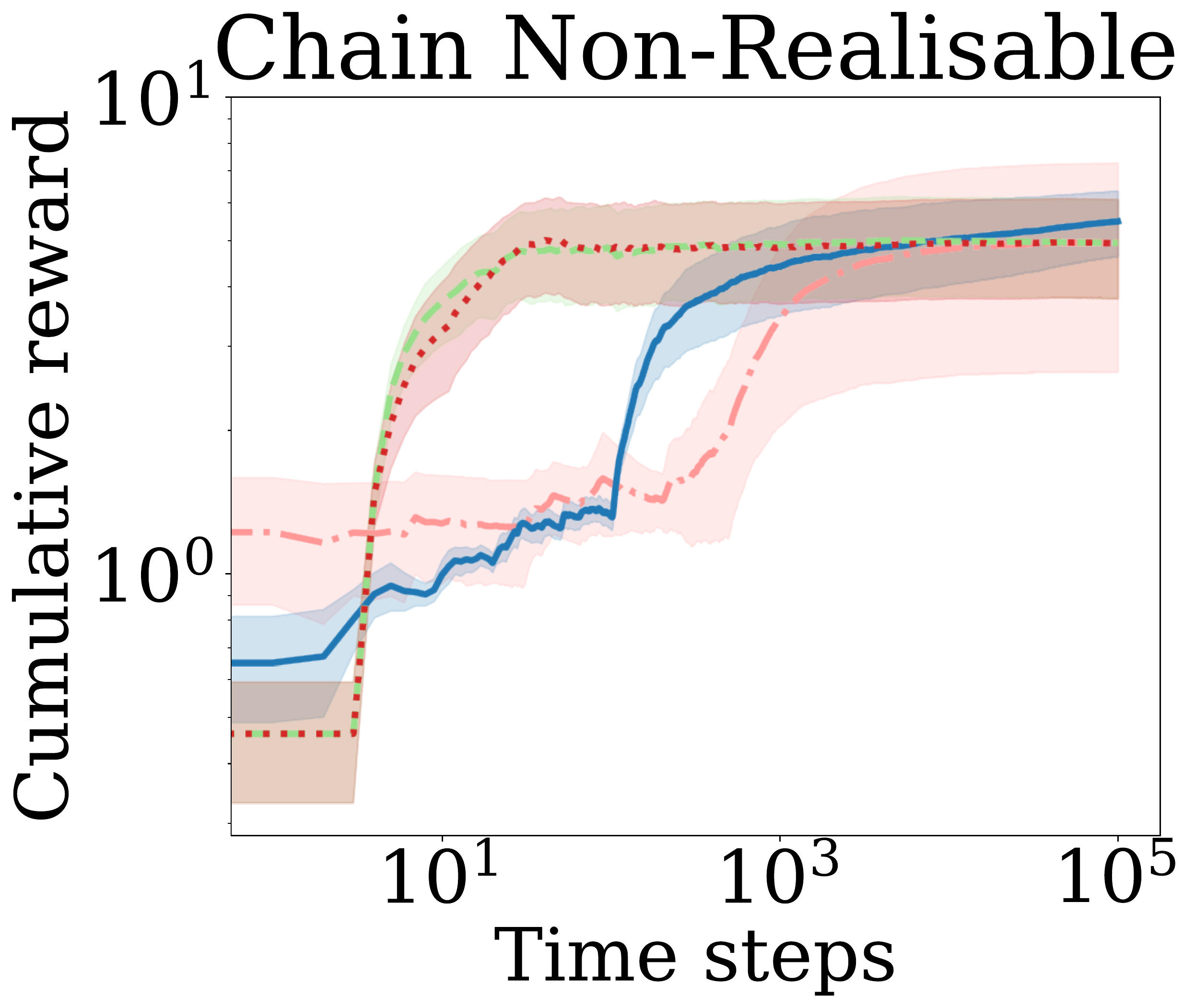}\\
    \includegraphics[width=0.96\textwidth]{img/lqr_legend.pdf}
    \caption{Performance of MLEMTRL, the meta algorithms, and the baselines for the case with known reward function in Chain. The y-axis is the average cumulative reward at every time step and the shaded region is the standard error.}\label{fig:known_reward_results}
\end{figure}

\end{document}